\documentclass[11pt]{article}
\usepackage{graphicx}
\newcount\Comments
\usepackage{bbm}
\usepackage{amsmath}
\usepackage{amssymb}
\usepackage{algorithm}
\usepackage[framemethod=TikZ]{mdframed}
\usepackage{mathtools}
\usepackage{amsthm}
\usepackage{natbib}
\usepackage{letltxmacro}
\usepackage{authblk}
\usepackage{centernot}

\PassOptionsToPackage{algo2e,ruled, linesnumbered}{algorithm2e}
\RequirePackage{algorithm2e}

\LetLtxMacro\amsproof\proof
\LetLtxMacro\amsendproof\endproof

\usepackage{thmtools}
\usepackage{xcolor}
\AtBeginDocument{%
  \LetLtxMacro\proof\amsproof
  \LetLtxMacro\endproof\amsendproof
}

\newtheorem{theorem}{Theorem}[section]
\newtheorem{proposition}[theorem]{Proposition}
\newtheorem{definition}{Definition}[section]
\newtheorem{lemma}[theorem]{Lemma}
\newtheorem{assumption}{Assumption}
\newtheorem{corollary}[theorem]{Corollary}

\newtheorem{question}[theorem]{Question}

\newtheorem{remark}{Remark}[section]
\newtheorem*{theorem*}{Theorem}

\usepackage{xpatch}
\xpatchcmd{\proof}{\itshape}{\normalfont\proofnamefont}{}{}

\usepackage[most]{tcolorbox}

\usepackage[letterpaper,top=2cm,bottom=2cm,left=3cm,right=3cm,marginparwidth=1.75cm]{geometry}

\usepackage[colorlinks=true, allcolors=blue]{hyperref}

\newcommand{\naturals}{\mathbb{N}}
\newcommand{\integers}{\mathbb{Z}}

\usepackage{color}
\definecolor{darkgreen}{rgb}{0,0.5,0}
\definecolor{purple}{rgb}{1,0,1}
\newcommand{\kibitz}[2]{\ifnum\Comments=1\textcolor{#1}{#2}\fi}

\newcommand{\ambuj}[1]  {\kibitz{darkgreen}   {[AT: #1]}}

\newcommand{\Acal}{\mathcal{A}}

\newcommand{\Dcal}{\mathcal{D}}

\newcommand{\Gcal}{\mathcal{G}}
\newcommand{\Hcal}{\mathcal{H}}
\newcommand{\Ical}{\mathcal{I}}

\newcommand{\Ocal}{\mathcal{O}}
\newcommand{\Tcal}{\mathcal{T}}
\newcommand{\Zcal}{\mathcal{Z}}
\newcommand{\Xcal}{\mathcal{X}}
\newcommand{\Ycal}{\mathcal{Y}}

\newcommand{\Lcal}{\mathcal{L}}

\DeclareMathOperator*{\argmax}{arg\,max}

\newcommand{\expect}{\operatorname{\mathbb{E}}}

\newcommand*\samethanks[1][\value{footnote}]{\footnotemark[#1]}

\title{Generation through the lens of learning theory}
\author{Jiaxun Li\thanks{Equal contribution}, Vinod Raman\samethanks, Ambuj Tewari}
\affil{Department of Statistics, University of Michigan}
\affil{\texttt{\{jasonli, vkraman, tewaria\}@umich.edu}}
\date{\today}

\begin{document}

\maketitle

\begin{abstract}
 We study generation through the lens of statistical learning theory. First, we abstract and formalize the results of \cite{gold1967language}, \cite{angluin1979finding}, \cite{angluin1980inductive} and \cite{kleinberg2024language} in terms of a binary hypothesis class defined over an abstract example space. Then, we extend the notion of ``generation" from \cite{kleinberg2024language} to two new settings, we call ``uniform" and ``non-uniform" generation, and provide a characterization of which hypothesis classes are uniformly and non-uniformly generatable. As is standard in learning theory, our characterizations are in terms of the finiteness of a new combinatorial dimension termed the Closure dimension. By doing so, we are able to compare generatability with predictability (captured via PAC and online learnability) and show that these two properties of hypothesis classes are \emph{incompatible} -- there are classes that are generatable but not predictable and vice versa. Finally, we extend our results to capture \emph{prompted} generation and give a complete characterization of which classes are prompt generatable, generalizing some of the work by \cite{kleinberg2024language}. 
\end{abstract}

\section{Introduction}
 Over the past 50 years, predictive machine learning has been a cornerstone for both theorists and practitioners. Predictive tasks like classification and regression have been extensively studied, in both theory and practice, due to their applications to face recognition, autonomous vehicles, fraud detection, recommendation systems, etc. Recently, however, a new paradigm of machine learning has emerged: \emph{generation}. Unlike predictive models, which focus on making accurate predictions of the true label given examples, generative models aim to \emph{create} new examples based on observed data. For example, in language modeling, the goal might be to generate coherent text in response to a prompt, while in drug development, one might want to generate candidate molecules. In fact, generative models have already been applied to these tasks and others \citep{zhao2023survey, jumper2021highly}.

The vast potential of generative machine learning has spurred a surge of research across diverse fields like natural language processing \citep{wolf2020transformers}, computer vision \citep{khan2022transformers}, and computational chemistry/biology \citep{vanhaelen2020advent}. Despite this widespread adoption, the theoretical foundations of generative machine learning lags far behind its predictive counterpart. While prediction has been extensively studied by learning theorists through frameworks like PAC and online learning \citep{ShwartzDavid, 10.5555/2371238, cesa2006prediction}, generative machine learning has, for the most, part remained elusive.  A key reason for this is that generation is fundamentally an \emph{unsupervised} task. Unlike classification or regression, where there is a true label or response to guide the model, generation lacks a clear notion of correctness. This makes it challenging to define a loss function -- the primary tool used in predictive tasks to quantify the quality of a model.

\subsection{Main Contributions}

In this work, we take a step towards closing this gap by bridging learning theory and generative machine learning. Our main contributions are four-fold.

\begin{itemize}
\item[(1)]  We formalize the notion of ``generation in the limit" from \cite{kleinberg2024language} through the lens of classical learning theory, allowing the tools and techniques that have been developed over the years to seamlessly port over. Specifically, we reinterpret the framework from \cite{kleinberg2024language} using a binary hypothesis class $\Hcal$ defined over a countable example space $\Xcal$. This abstraction generalizes beyond language generation, enabling our results to apply to other generative tasks such as image and molecule generation.

\item[(2)] We introduce two new paradigms of generation called ``uniform" and ``non-uniform`` generation, both of which are stronger than the notion of ``generation in the limit" from \cite{kleinberg2024language}. While \cite{kleinberg2024language} do show that finite hypothesis classes are uniformly generatable, they leave the full characterization of uniform generatability open. We close this gap and provide a complete characterization of which hypothesis classes are uniformly generatable in terms of a new combinatorial dimension we term the Closure dimension. 

\begin{theorem*}[Informal]A class $\Hcal \subseteq \{0, 1\}^{\Xcal}$ is \emph{uniformly generatable} if and only if $\operatorname{C}(\Hcal) < \infty$, where  $\operatorname{C}(\Hcal)$ is the \emph{Closure dimension} of $\Hcal$.
\end{theorem*}

In addition, we use the Closure dimension to fully characterize which classes are non-uniformly generatable. 

\begin{theorem*}[Informal] A class $\Hcal \subseteq \{0, 1\}^{\Xcal}$ is \emph{non-uniformly generatable} if and only if there exists a non-decreasing sequence of classes $\Hcal_1 \subseteq \Hcal_2 \subseteq \dots$ such that $\Hcal = \bigcup_{n=1}^{\infty} \Hcal_n$ and $\operatorname{C}(\Hcal_n) < \infty$ for every $n \in \naturals$.
\end{theorem*}

In fact, while \cite{kleinberg2024language} show that all countable classes are generatable in the limit, our characterization of non-uniform generatability shows that all countable classes are non-uniformly generatable. This provides an improvement as non-uniform generation is a strictly harder than generation in the limit. With respect to generatability in the limit, we provide an alternate sufficiency condition in terms of the Closure dimension which, in conjunction with countableness, significantly expands the collection of classes that are known to be generatable in the limit.

\begin{theorem*}[Informal] A class $\Hcal \subseteq \{0, 1\}^{\Xcal}$ is \emph{generatable in the limit} if there exists a  finite sequence of classes $\Hcal_1, \Hcal_2, \dots, \Hcal_n$ such that $\Hcal = \bigcup_{i=1}^{n} \Hcal_i$ and $\operatorname{C}(\Hcal_i) < \infty$ for all $i \in [n].$
\end{theorem*}

In addition to the above theorem, we  give two other sufficiency conditions for generatability in the limit in terms of what we call the ``Eventually Unbounded Closure" property. As a corollary, we are able to recover the results of \cite{kleinberg2024language} that all countable classes are generatable in the limit. We leave the complete characterization of generatability in the limit as an important future question (see Section \ref{sec:diss}).

\item[(3)] By adopting a learning theory perspective, we uncover fundamental differences between generation and \emph{prediction} for countable classes, the latter of which we measure through PAC and online learnability. In particular, we find that these two tasks are incompatible -- there exist countable hypothesis classes for which generation is possible but prediction is not, and vice versa. In addition, unlike prediction, we find generation behaves very poorly under unions. Overall, these results highlight that the generative paradigm is distinct from prediction and suggests that new tools must be developed to fully understand the capabilities and limitations of generative machine learning.

\item[(4)] We define and extend our results to capture a notion of \emph{prompted} generation,  generalizing some of the results from Section 7 of \cite{kleinberg2024language}. By extending the Closure dimension to the Prompted Closure dimension, we prove identical characterizations of prompted uniform and non-uniform generatability as in the informal theorems in Contribution (2).
\end{itemize}

Our results extend the study of generation beyond language modeling, but are mainly information-theoretic in nature. That said, for all our algorithms, we point out computational primitives that are necessary for computability.




\subsection{Related Works}
The literature on generative machine learning is too vast to be included in complete detail. Thus, we refer the reader to the books and surveys by \cite{jebara2012machine}, \cite{harshvardhan2020comprehensive}, and \cite{pml2Book}.

The works most related to our work are those regarding language identification and generation in the limit \citep{gold1967language, angluin1979finding, angluin1980inductive}. Instead of an example space $\Xcal$ and a hypothesis class $\Hcal \subseteq \{0, 1\}^{\Xcal}$, these works consider a countable set $U$ of strings and a countable language family $C = \{L_1, L_2, \dots\}$, where $L_i \subset U$ for all $i \in \mathbbm{N}$.

In the Gold-Angluin model, an adversary first picks a language $K \in C$, and begins to enumerate the strings one by the one to the player in rounds $t = 1, 2, \dots$. After observing the string $w_t$ in round $t \in \mathbbm{N}$, the player guesses a language $L_t$ in $C$ with the hope that $L_t = K$. Crucially, the player gets no feedback at all. The player has identified $K$ in the limit, if there exists a finite time step $t^{\star} \in \mathbbm{N}$ such that for all $s \geq t$, we have that $L_s = K$.

In full generality, \cite{gold1967language} showed that language identification in the limit is impossible -- there are simple language families $C$, like those produced by finite automata, for which no algorithm can perform language identification in the limit. Following this work,  \cite{angluin1979finding, angluin1980inductive} provide a precise characterization of which language families $C$ is language identification in the limit possible. This characterization further emphasized the impossibility of language identification in the limit by ruling out the vast majority of language families. 

Very recently, and inspired by large language models, \cite{kleinberg2024language} study the problem of language \emph{generation} in the limit. In this problem, the adversary also  picks a language $K \in C$, and begins to enumerate the strings one by the one to the player in rounds $t = 1, 2, \dots$. However, now, after observing the string $w_t$ in round $t \in \mathbbm{N}$, the player guesses a string $\hat{w}_t \in U$ with the hope that $\hat{w}_t \in K \setminus \{w_1, \dots, w_t\}$. Once again, the player gets no feedback at all. The player has generated from $K$ in the limit, if there exists a finite time step $t \in \mathbbm{N}$ such that for all $s \geq t$, we have that $\hat{w}_s \in K \setminus \{w_1, \dots, w_s\}$. Remarkably, \cite{kleinberg2024language} prove a strikingly different result -- while Gold-Angluin show that identification in the limit is impossible for most language families, \cite{kleinberg2024language} show that generation in the limit is possible for \emph{every} countable language family $C$. This shows that language identification and generation are drastically different in the limit. 

Concurrently and independently from our work, \cite{kalavasis2024limits} study generation in the stochastic setting, where the positive examples revealed to the generator are sampled i.i.d. from some unknown distribution. In this model, they study the trade-offs between generating with breadth and generating with consistency and resolve the open question posed by \cite{kleinberg2024language} for a large family of language models. In addition, \cite{kalavasis2024limits} quantify the error rates for generation with breadth/consistency according to the universal rates framework of \cite{bousquet2021theory}.

In a previous version of this manuscript, we had posed whether all countable class are non-uniformly generatable (see Definition \ref{def:generability}) as an open question. In a recent follow-up work, \cite{charikar2024exploring}, independently of us, resolve this affirmatively. However, \cite{charikar2024exploring} show that non-uniform generation is not possible using only membership queries. This is in contrast to generatability in the limit, where \cite{kleinberg2024language} show that every countable classes is generatable in the limit using only membership queries. In addition, similar to \cite{kalavasis2024limits}, \cite{charikar2024exploring} establish a trade-off between consistency and breadth by defining a new notion of exhaustive generation and provide a complete characterization of which classes are exhaustively generatable. Finally, \cite{charikar2024exploring} characterize which classes are uniformly generatable when feedback is available. 

\section{Preliminaries} \label{sec:prelim}

\subsection{Notation} \label{sec:not}
Let $\Xcal$ denote a \emph{countable} example space (e.g. text, molecules, images) and $\Hcal \subseteq \{0, 1\}^{\Xcal}$ denote a binary hypothesis class (e.g. class of all vision transformers). Let $\Xcal^{\star}$ denote the set of all finite subsets of $\Xcal$. In the context of language modeling, one can think about $\Xcal$ as the collection of all valid strings, and each hypothesis $h \in \Hcal$ as a language (i.e. a subset of strings). For a hypothesis $h \in \Hcal$, let $\text{supp}(h) := \{x \in \Xcal: h(x) = 1\}$, that is, its collection of positive examples. For any $h \in \Hcal$, an \emph{enumeration} of $\operatorname{supp}(h)$ is any infinite sequence $x_1, x_2, \dots$ such that $\bigcup_{i \in \mathbbm{N}} \{x_i\} = \operatorname{supp}(h)$. In other words, for every $x \in \operatorname{supp}(h)$, there exists an $i \in \mathbbm{N}$ such that $x_i = x$. 

For any class $\Hcal$ and a finite sequence of examples $x_1, \dots, x_n$, let

$$\Hcal(x_1, \dots, x_n) := \{h \in \Hcal: \{x_1, \dots, x_n\} \subseteq \operatorname{supp}(h)\}.$$ In learning theory literature, $\Hcal(x_1, \dots, x_n)$ is also called the ``version space" of  $\Hcal$ induced by the sample $\{(x_i, 1)\}_{i=1}^n$ (i.e., the set of all consistent hypothesis).

For any class $\Hcal$, define $\langle \cdot \rangle_{\Hcal}$ as its induced closure operator such that

$$\langle x_1,\dots,x_n \rangle_{\Hcal} := \begin{cases}
			\bigcap_{h \in \Hcal(x_{1:n})} \operatorname{supp}(h), & \text{if $|\Hcal(x_{1:n})| \geq 1$}\\
            \bot, & \text{if $|\Hcal(x_{1:n})| = 0$}
		 \end{cases}.$$

\vspace{5pt}

\begin{remark} \label{rem:erm}In learning-theoretic terms, $\langle x_1,\dots,x_n \rangle_{\Hcal}$ is the set of positive examples common to all hypothesis in the version space of $\Hcal$ consistent with the sample $(x_1, 1), \dots, (x_n, 1).$ From this perspective, one can check closure membership, i.e. given an example $x$ and a sequence of examples $x_1, \dots, x_n$, return  $\mathbbm{1}\{\langle x_1,\dots,x_n \rangle_{\Hcal} \neq \bot \text{ and } x\in \langle x_1,\dots,x_n \rangle_{\Hcal} \}$, using access to an Empirical Risk Minimization \emph{(ERM)} oracle. Formally, an \emph{ERM} oracle is a mapping $\Ocal: 2^{\{0, 1\}^{\Xcal}} \times (\Xcal \times \{0, 1\})^{\star} \rightarrow \naturals \cup \{0\}$, which given a class $\Hcal \subseteq \{0, 1\}^{\Xcal}$ and a labeled sample $S \in (\Xcal \times \{0, 1\})^{\star}$, outputs $\min_{h \in \Hcal} \sum_{(x, y) \in S} \mathbbm{1}\{h(x) \neq y\}.$ Then, given a class $\Hcal$, a sequence of examples $x_1, \dots, x_n$,  one can compute  $\mathbbm{1}\{\langle x_1,\dots,x_n \rangle_{\Hcal} \neq \bot \text{ and } x\in\mathbbm{1}\{\langle x_1,\dots,x_n \rangle_{\Hcal} \}$ using the following procedure. First, pass to $\Ocal$ the sample $S = \{(x_1,1), \dots, (x_n, 1)\}$ and $\Hcal$, and let $r$ be its output. If $r \geq 1$, output $0$.  Otherwise, define the sample $S_x = \{(x_1, 1), \dots, (x_n, 1), (x, 0)\}.$ Query $\Ocal$ on $S_x$ and $\Hcal$ and let $r_x$ be its output. Output $r_x$. To see why the latter step works, suppose $r_x = 0$. Then, that means there exists a hypothesis $h \in \Hcal$ such that $\{x_1, \dots. x_n\} \subseteq \operatorname{supp}(h)$ but $x \notin \operatorname{supp}(h)$. Thus, it cannot be the case that $x \in \langle x_1, \dots, x_n \rangle_{\Hcal}.$ On the other hand, if $r_x = 1$, then it must mean that for every $h \in \Hcal$ such that $\{x_1, \dots, x_n\} \subseteq \operatorname{supp}(h)$, we have that $h(x) = 1$. Accordingly, $x \in \langle x_1, \dots, x_n \rangle_{\Hcal}$ by definition. 
\end{remark}


\vspace{5pt}

\noindent We will sometimes make the following assumption about hypothesis classes.

\begin{assumption} [Uniformly Unbounded Support (UUS)] A hypothesis class $\Hcal \subseteq \{0, 1\}^{\Xcal}$ satisfies the \emph{Uniformly Unbounded Support (UUS)} property if $|\operatorname{supp}(h)| = \infty$ for every $h \in \Hcal$. 
\end{assumption}

Finally, we will let $[N]:= \{1, \dots, N\}$ and abbreviate a finite sequence $x_1, \dots, x_n$ as $x_{1:n}.$

\subsection{Generatability}
We restate the model of generation in \cite{kleinberg2024language} using the notation in Section \ref{sec:not}. Consider the following two-player game. At the start, the adversary picks a hypothesis $h \in \Hcal$ and an enumeration $x_1, x_2, \dots$ of $\operatorname{supp}(h)$ and does not reveal them to the learner. The game then proceeds over rounds $t = 1, 2, \dots$. In each round $t \in \mathbbm{N}$, the adversary reveals $x_t$. The generator, after observing $x_1, \dots, x_t$, must output $\hat{x}_t \in \Xcal \setminus \{x_1, \dots, x_t\}$ and suffers the loss  $\mathbbm{1}\{\hat{x}_t \notin \operatorname{supp}(h) \setminus \{x_1, \dots, x_t\} \}$. Crucially, the generator \emph{never} observes its loss as it does not know $h$. The goal of the generator is to eventually perfectly generate new, positive examples $x \in \operatorname{supp}(h)$.

\begin{remark}
    We highlight that the generator does not receive any feedback. This captures the unsupervised nature of generative machine learning. 
\end{remark}

To make this formal, we first define a generator as  a mapping from a finite sequence of examples to a new example.

\begin{definition}[Generator] \label{def:generator} A generator is a map $\Gcal: \Xcal^{\star} \rightarrow \Xcal$ that takes a finite sequence of examples $x_1, x_2, \dots$ and outputs a new example $x$. 
\end{definition}

We can now use the existence of a good generator to define the property of generatability in the limit from Section 2 of \cite{kleinberg2024language}.

\begin{definition}[Generatability in the Limit] \label{def:geninlim} Let $\Hcal \subseteq \{0, 1\}^{\Xcal}$  be any hypothesis class satisfying the $\operatorname{UUS}$ property. Then, $\Hcal$ is \emph{generatable in the limit} if there exists a generator $\Gcal$ such that for every $h \in \Hcal$, and any \emph{enumeration} $x_1, x_2, \dots$ of $\operatorname{supp}(h)$, there exists a  $t^{\star} \in \mathbbm{N}$ such  that $\Gcal(x_{1:s}) \in \operatorname{supp}(h) \setminus \{x_1, \dots, x_s\}$ for all $s \geq t^{\star}$.
\end{definition}


Roughly speaking, generatability captures the \emph{existence} of the ability to eventually generate positive examples, when no feedback is available and the underlying hypothesis is not known. Note that the adversary can repeat examples in its stream, but it must eventually enumerate all the positive examples of the chosen hypothesis. On the other hand, the adversary is still powerful as it can examine/simulate the generator $\Gcal$ in any way imaginable before choosing the true hypothesis and the enumeration of its support.

\begin{remark}
\cite{kleinberg2024language} assume that their algorithms have  access to ``membership queries." That is, their algorithms have access to black box that can answer questions of the form ``Is $w \in L$?" for any string $w \in U$. In our model, this is equivalent to being able to pick an $h \in \Hcal$ and query its label on any $x \in \Xcal$, which is the standard assumption in learning theory literature. In this paper, we are not concerned with computability. Accordingly, a generator is simply a function that maps past examples to a future one. This view of generators aligns with Section 4 of \cite{kleinberg2024language}, where the authors do not place any computational restrictions on their algorithms.
\end{remark}

While not explicitly defined, \cite{kleinberg2024language} also consider a notion of \emph{uniform} generatability in Theorem 2.2 of the Section titled ``A Result for Finite Collections." By ``uniform", we mean that the amount of time required before the generator should perfectly generate new positive examples should only be a function of the class $\Hcal$ and thus the same across all hypothesis $h \in \Hcal$ and enumerations of $\operatorname{supp}(h)$.  This is in contrast to generatability in the limit, where the time step $t^{\star}$ can depend on both the sequence of examples $x_1, x_2, \dots$ and the selected hypothesis $h \in \Hcal$. Definition \ref{def:unifgen} formalizes this notion of uniform generatability.


\begin{definition}[Uniform Generatability] \label{def:unifgen} Let $\Hcal \subseteq \{0, 1\}^{\Xcal}$  be any hypothesis class satisfying the $\operatorname{UUS}$ property. Then, $\Hcal$ is \emph{uniformly} generatable, if there exists a generator $\Gcal$ and $d^{\star} \in \mathbbm{N}$, such that for every $h \in \Hcal$ and any sequence $x_1, x_2, \dots$ with $\{x_1, x_2, \dots\} \subseteq \operatorname{supp}(h)$, if there exists $t^{\star} \in \mathbbm{N}$ such that $|\{x_1, \dots, x_{t^{\star}}\}| = d^{\star}$, then $\Gcal(x_{1:s}) \in  \operatorname{supp}(h) \setminus \{x_1, \dots, x_s\}$ for all $s \geq t^{\star}$.
\end{definition}

A subtle detail in Definition \ref{def:unifgen} is the fact that we must force the adversary to play a sufficient number of distinct examples before we require the generator to be perfect. This is necessary, as otherwise, the adversary can play the same example in all rounds, and even a simple hypothesis class with two hypotheses which share exactly one example in their support cannot be uniformly generatable. This restriction is also captured by Theorem 2.2 in \cite{kleinberg2024language} and effectively means that a generator witnessing Definition \ref{def:unifgen} must be able to generate new, positive examples after observing a sufficient number of distinct positive examples. In fact, given a generator $\Gcal$, the number of positive examples needed before perfect generation is akin to sample complexity and mistake-bounds in PAC and online learning. This motivates a notion of ``uniform generation sample complexity."

\begin{definition}[Uniform Generation Sample Complexity] Given a class $\Hcal \subseteq \{0, 1\}^{\Xcal}$ and a generator $\Gcal$, the \emph{uniform generation sample complexity} of a generator $\Gcal$ is the smallest number $d_{\Gcal} \in \mathbbm{N}$, such that $\Gcal$ perfectly generates according to Definition \ref{def:unifgen} after it observes $d_{\Gcal}$ unique positive examples. If no such number exists, we set $d_{\Gcal} = \infty.$
\end{definition}

While natural, one can also arrive at uniform generatability by swapping the order of quantifiers in Definition \ref{def:geninlim}. More precisely, one also gets uniform generatability by moving generation sample complexity ``left" twice. That is, before the quantifiers on the selected hypotheses and the stream, and therefore, in line with the the existence of the generator. This motivates an \emph{intermediate} setting, we term non-uniform generatability, where we only move the generation sample complexity ``left" once, and in particular, before the quantifier on the stream.

\begin{definition}[Non-uniform Generatability] \label{def:generability} Let $\Hcal \subseteq \{0, 1\}^{\Xcal}$ be any hypothesis class satisfying the $\operatorname{UUS}$ property. Then, $\Hcal$ is \emph{non-uniformly generatable} if there exists a generator $\Gcal$ such that for every $h \in \Hcal$, there exists a  $d^{\star} \in \mathbbm{N}$ such that for any sequence $x_1, x_2, \dots$ with $\{x_1, x_2, \dots\} \subseteq \operatorname{supp}(h)$, if there exists $t^{\star} \in \mathbbm{N}$ such that $|\{x_1, \dots, x_{t^{\star}}\}| = d^{\star}$, then $\Gcal(x_{1:s}) \in \operatorname{supp}(h) \setminus \{x_1, \dots, x_s\}$ for all $s \geq t^{\star}$.
\end{definition}

 We use the term ``non-uniform" to denote the fact that the number of distinct examples needed before perfect generation can depend on the hypothesis selected by the adversary, and hence, it is ``non-uniform" over the hypothesis class $\Hcal$. But do note that the number of distinct examples needed is still uniform over the possible stream chosen by the adversary. Hence, we use the term ``uniform" and ``non-uniform" only with respect to the hypothesis chosen by the adversary. Again, we require the restriction that the  adversary must select a sufficient number of distinct examples, as otherwise, even trivial classes are not non-uniformly generatable.

By inspecting the order of quantifiers, it is clear that uniform generatability is the strongest of the three properties, while generatability in the limit is the weakest.  In particular, for any class $\Hcal$, we have that 
$$\text{Uniform Generatability} \implies \text{Non-uniform Generatability} \implies \text{Generatability in the limit}.$$

This ordering is tight in the sense that the reverse directions are \emph{not} true.

\begin{proposition} \label{prop:gencomp} Let $\Xcal$ be countable. There exists classes $\Hcal_1, \Hcal_2 \subseteq \{0, 1\}^{\Xcal}$ satisfying the \emph{UUS} property such that: 

\begin{itemize}
\item[(i)] $\Hcal_1$ is non-uniformly generatable but not uniformly generatable.
\item[(ii)] $\Hcal_2$ is generatable in the limit but not non-uniformly generatable. 
\end{itemize}
    
\end{proposition}

We will prove Proposition \ref{prop:gencomp} in Section \ref{sec:charunifgen}. We end this section by highlighting an important practical property of uniform and non-uniform generators. Our definitions of uniform and non-uniform generatability do not require the adversary to select an enumeration of the support of its selected hypothesis. That is, any valid sequence with a sufficient number of distinct examples is enough. As a consequence, once enough distinct examples are revealed to the generator, the adversary can reveal the generators prediction on round $t$ as the positive example on round $t+1$. This effectively means that once the  generator has observed a sufficient number of distinct example, it needs to be able to produce new, unseen positive examples \emph{auto-regressively} and with no feedback at all. One way to interpret this is that once enough distinct positive examples are revealed to the generator $\Gcal$, it is able to continuously produce new unseen positive examples on its own  without any supervision by setting $x_{t+1} = \Gcal(x_1, \dots, x_t).$ This property might be useful in applications where the generator is used for downstream tasks.

\subsection{Identifiability}

In identification, one seeks not to output new, positive examples $x \in \Xcal$, but rather, to identify the true, underlying hypothesis $h \in \Hcal$ chosen by the adversary. Historically, identification has been studied in the context of language modeling, with works dating as far back as Gold's seminal work on language identification in the limit \citep{gold1967language}. For consistency sake, we will formally define Gold's model in the notation of this paper. As in generation, we start by defining an Identifier.

\begin{definition}[Identifier] An Identifier is a map $\Ical:  \Xcal^{\star} \rightarrow \{0, 1\}^{\Xcal}$ that takes as input a finite sequence of examples $x_1, x_2, \dots$ and outputs a hypothesis. 
\end{definition}

The notion of identifiability in the limit can now be written in terms of the existence of good identifiers, and one can verify that our definition of identifiability in the limit is equivalent to that from \cite{gold1967language} and \cite{angluin1979finding, angluin1980inductive}.

\begin{definition}[Identifiability in the limit] \label{def:ident} Let $\Hcal \subseteq \{0, 1\}^{\Xcal}$  be any hypothesis class satisfying the $\operatorname{UUS}$ property. Then, $\Hcal$ is \emph{identifiable in the limit} if there exists a identifier $\Ical$ such that for every $h \in \Hcal$ and any enumeration $x_1, x_2, \dots$  of $\operatorname{supp}(h)$, there exists a  $t^{\star} \in \mathbbm{N}$ such that $\Ical(x_{1:s}) = h$ for all $s \geq t^{\star}$.
\end{definition}



Although analogous definitions of uniform and non-uniform identifiability exist,   we do not define or focus on them here as they are stronger than identifiability in the limit, which is already a very restrictive requirement.




\subsection{Predictability}
It is also natural to understand the predictability of a hypothesis class $\Hcal$. Informally, the predictability of a class $\Hcal$ should measure how easy it is to predict the labels of new examples $x_1, x_2, \dots$ when the labels are produced by some unknown hypothesis $h \in \Hcal$. In this paper, we will measure predictability of a hypothesis class $\Hcal$ through their PAC and online \emph{learnability} -- properties of hypothesis classes that have been extensively studied by learning theorists \citep{vapnik1974theory, Littlestone1987LearningQW, ben2009agnostic}.

In the PAC learning model, an adversary picks both a distribution $\Dcal$ over $\Xcal$ and a hypothesis $h \in \Hcal$.  The learner receives $n$ iid samples $S = \{x_i, h(x_i)\}_{i=1}^n \sim (\Dcal \times h)^n$, where we use $\Dcal \times h$ to denote the distribution over $\Xcal \times \{0, 1\}$ defined procedurally by first sampling $x \sim \Dcal$ and then outputting $(x, h(x))$. The goal of the learner is to use the sample $S$ to output a hypothesis $f \in \{0, 1\}^{\Xcal}$ such that $f$ has low error probability on a \emph{future} labeled example drawn from $\Dcal$. 

\begin{definition}[PAC Learnability]
A hypothesis class $\Hcal$ is PAC  learnable, if there exists a function $m:(0,1)^2 \to \naturals$ and a learning algorithm $\Acal: (\Xcal \times \{0, 1\})^{\star} \to \{0, 1\}^{\Xcal}$ with the following property:  for every $\epsilon, \delta \in (0, 1)$,  distribution $\Dcal $ on $\Xcal$, and $h \in \Hcal$, algorithm $\Acal$ when run on $n \geq m(\epsilon, \delta)$ iid samples $S = \{(x_i, h(x_i))\}_{i = 1}^{n} \sim (\Dcal \times h)^n$, outputs a predictor $f := \Acal(S) \in \{0, 1\}^{\Xcal}$ such that with probability at least $ 1-\delta$ over $S \sim (\Dcal \times h) ^{n}$,
\[\expect_{x \sim \Dcal}[\mathbbm{1}\{f(x) \neq h(x)\}] \leq \epsilon.\]
\end{definition}

The seminal result by \cite{vapnik1971uniform} shows that the finiteness of the Vapnik–Chervonenkis (VC) dimension characterizes which hypothesis classes are PAC learnable.  

\begin{definition}[VC dimension]
A sequence $(x_1, \dots, x_d) \in \mathcal{X}^d$ is shattered by $\mathcal{H}$, if $\, \forall \, (y_1, \dots, y_d) \in \{0, 1\}^d$, $\exists h\in \mathcal{H}$, such that  $\forall i \in [d]$, $h(x_i) = y_i$. The VC dimension of $\mathcal{H}$, denoted $\emph{\text{VC}}(\mathcal{H})$, is the largest number $d \in \mathbbm{N}$ such that there exists a sequence $(x_1, \dots, x_d) \in \mathcal{X}^d$ that is shattered by $\mathcal{H}$. If there exists shattered sequences of arbitrarily large length $d \in \mathbbm{N}$, then we say that $\operatorname{VC}(\Hcal) = \infty$.
\end{definition}

In the online learning model, no distributional assumptions are placed \citep{Littlestone1987LearningQW, ben2009agnostic}. Instead,  an adversary plays a sequential game with the learner over $T \in \mathbbm{N}$ rounds. Before the game begins, the adversary selects a sequence of examples $x_1, x_2, \dots, x_T$ and a hypothesis $h \in \Hcal$. Then, in each round $t \in [T]$, the reveals first reveals $x_t$ to the learner, the learner makes a  prediction $\hat{y}_t \in \{0, 1\}$, the adversary reveals the true label $h(x_t)$, and finally the learner suffers the loss $\mathbbm{1}\{\hat{y}_t \neq h(x_t)\}$.The goal of the learner is to output predictions $\hat{y_t}$ such that its cumulative number of mistakes is ``small."

\begin{definition} [Online Learnability]
\label{def:agnOL}
A hypothesis class $\Hcal$ is online learnable if there exists an  algorithm $\mathcal{A}$ and sublinear function $\operatorname{R}: \mathbbm{N} \rightarrow \mathbbm{N} $ such that for any $T \in \mathbbm{N}$, any sequence of examples $x_1, \dots, x_T$, and any $h \in \Hcal$, the algorithm outputs $\hat{y}_t \in \{0, 1\}$ at every time point $t \in [T]$ such that 
$$\sum_{t=1}^T \mathbbm{1}\{\hat{y}_t \ne h(x_t) \}\leq \operatorname{R}(T).$$
\end{definition}

The online learnability of a hypothesis class $\Hcal$ is characterized by the finiteness of a different combinatorial parameter called the Littlestone dimension \citep{Littlestone1987LearningQW}. To define the Littlestone dimension, we first need to define a Littlestone tree and an appropriate notion of shattering. 

\begin{definition}[Littlestone tree]
\noindent A Littlestone tree of depth $d$ is a \emph{complete} binary tree of depth $d$ where the internal nodes are labeled by examples of $\mathcal{X}$ and the left and right outgoing edges from each internal node are labeled by $0$ and $1$ respectively.
\end{definition}

Given a Littlestone tree $\mathcal{T}$ of depth $d$, a root-to-leaf path down $\mathcal{T}$ is a bitstring $\sigma \in \{0, 1\}^d$ indicating whether to go left ($\sigma_i = 0$) or to go right ($\sigma_i = 1$) at each depth $i \in [d]$. A path $\sigma \in \{0, 1\}^d$ down $\mathcal{T}$ gives a sequence of labeled examples $\{(x_i, \sigma_i)\}_{i = 1}^{d}$, where $x_i$ is the example labeling the internal node following the prefix $(\sigma_1, \dots, \sigma_{i-1})$ down the tree. A hypothesis $h_{\sigma} \in \mathcal{H}$ shatters a path $\sigma \in \{0, 1\}^d$, if for every $i \in [d]$, we have  $h_{\sigma}(x_i) = \sigma_i$. In other words, $h_{\sigma}$ is consistent with the labeled examples when following $\sigma$. A Littlestone tree $\mathcal{T}$ is \emph{shattered} by $\mathcal{H}$ if for every root-to-leaf path $\sigma$ down $\mathcal{T}$, there exists a hypothesis $h_{\sigma} \in \mathcal{H}$ that shatters it. Using this notion of shattering, we define the Littlestone dimension of a hypothesis class. 

\begin{definition}[Littlestone dimension]
\noindent The Littlestone dimension of $\mathcal{H}$, denoted $\operatorname{L}(\mathcal{H})$, is the largest $d \in \mathbbm{N}$ such that there exists a Littlestone tree $\mathcal{T}$ of depth $d$ shattered by $\mathcal{H}$. If there exists shattered Littlestone trees $\mathcal{T}$ of arbitrary large depth, then we say that $\operatorname{L}(\mathcal{H}) = \infty$. 
\end{definition}

It is well known that online learnability implies PAC learnability, but not the other way around. That is, for every $\Hcal \subseteq \{0, 1\}^{\Xcal}$, we have that $\operatorname{L}(\Hcal) \geq \operatorname{VC}(\Hcal)$, and that the inequality can be strict. Our definitions of PAC and online learnability are uniform in nature, as is standard in learning theory literature \citep{ShwartzDavid}. There are also non-uniform and ``in-the-limit" versions of PAC and online learnability. However, we will not be concerned with them in this paper and will not make explicit the distinction between uniform and non-uniform predictability. We refer the reader Chapter 7 in \cite{ShwartzDavid} and \cite{lu2023non} for more details about the non-uniform versions of PAC and online learnability respectively and \cite{malliaris2022unstable} for a ``in-the-limit" version of PAC learnability (termed PEC learnability).


\subsection{Existing Results in Identification and Generation}
In this section, we restate the results for language identification and generation in learning theory notation. \cite{gold1967language},  \cite{angluin1979finding}, and \cite{angluin1980inductive} studied the problem of identification in the context of language modeling. In our notation, they showed that many natural hypothesis classes are \emph{not} identifiable in the limit according to Definition \ref{def:ident}. This result is often interpreted as a hardness result -- identification in the limit is \emph{impossible} in full generality, even for some natural countable classes. 

\begin{theorem}[\cite{gold1967language, angluin1979finding, angluin1980inductive}] Let $\Xcal$ be countable. There exists a countable $\Hcal \subseteq \{0, 1\}^{\Xcal}$ which is not identifiable in the limit.
\end{theorem}

In particular, \cite{angluin1980inductive} provides a precise characterization of which classes are identifiable in the limit. Roughly, the condition states that every language $L$ must have a ``tell-tale" finite subset of strings $S \subset L$ such that any other language  $L^{\prime}$ which also contains $S$ cannot be a proper subset of $L$.  Theorem \ref{thm:idenhard} restates this condition in the notation of this paper.

\begin{theorem}[Theorem 1 in \cite{angluin1980inductive}]  \label{thm:idenhard} Let $\Xcal$ be countable and $\Hcal \subseteq \{0, 1\}^{\Xcal}$ be any hypothesis class. Then $\Hcal$ is identifiable in the limit if and only if for every $h \in \Hcal$, there exists $S \subseteq \operatorname{supp}(h)$ such that: 

\begin{itemize}
\item[(i)] $|S| < \infty$.
\item[(ii)]  $\forall h^{\prime} \in \Hcal$, $S \subseteq \operatorname{supp}(h^{\prime}) \implies \operatorname{supp}(h^{\prime}) \not\subset\operatorname{supp}(h) .$
\end{itemize}

\end{theorem}

On the other hand, \cite{kleinberg2024language} recently show that this is not the case for generatability in the limit -- all \emph{countable} $\Hcal$ that satisfy the UUS property are generatable in the limit! 

\begin{theorem} [Theorem 4.1 in \cite{kleinberg2024language}] \label{thm:gencount} Let $\Xcal$ be countable and $\Hcal \subseteq \{0, 1\}^{\Xcal}$. If $\Hcal$ is countable and satisfies the $\operatorname{UUS}$ property, then $\Hcal$ is generatable in the limit. 
\end{theorem}

In addition, they also prove that finite hypothesis classes satisfy the stronger notion of uniform generatability. 

\begin{theorem} [Theorem 2.2 in \cite{kleinberg2024language}] 
\label{thm:genfinite}
Let $\Xcal$ be countable and $\Hcal \subseteq \{0, 1\}^{\Xcal}$. If $\Hcal$ is finite and satisfies the $\operatorname{UUS}$ property, then $\Hcal$ is \emph{uniformly} generatable. 
\end{theorem}

Unfortunately, \cite{kleinberg2024language} do not give a full characterization of which classes are uniformly generatable, non-uniformly generatable, and generatable in the limit. 
In this paper, we are interested in closing these gaps. In particular, we are interested in identifying  necessary and sufficient conditions under which a hypothesis class $\Hcal$ is uniformly and non-uniformly generatable. In learning theory, such conditions are often derived in terms of \emph{combinatorial dimensions}, which are mappings
$$\operatorname{dim}: 2^{\{0, 1\}^{\Xcal}} \rightarrow \mathbbm{N} \cup \{0, \infty\}$$
\noindent such that $\operatorname{dim}(\Hcal)$ measures an appropriate notion of expressivity of a class $\Hcal$. For example, the PAC learnability of a binary hypothesis class is fully characterized by the finiteness of the VC dimension \citep{vapnik1974theory}. Similarly, the online learnability of a binary hypothesis class is characterized by the finiteness of its Littlestone dimension \citep{Littlestone1987LearningQW}. Are there analogous dimensions that characterize uniform,  non-uniform generatability, and generatability in the limit? These questions are the main focus of this paper. 


\section{Characterizations of Generatability} \label{sec:charunifgen}

In this section, we provide a characterization of which classes are uniformly and non-uniformly generatable, as well as, a weaker sufficiency condition for generatability in the limit. We start with characterizing uniform generation.

\subsection{Uniform Generatability}
  In learning theory, it is often the case that the most ``obvious" necessary condition is also sufficient. To that end, we seek a combinatorial dimension of $\Hcal$  whose infiniteness implies that $\Hcal$ is \emph{not uniformly generatable}.  By inverting Definition \ref{def:unifgen}, we have that $\Hcal$ is \emph{not} uniformly generatable if for every generator $\Gcal$ and every $d \in \mathbbm{N}$, there exists a $h^{\star} \in \Hcal$ and a sequence $(x_i)_{i \in \mathbbm{N}}$ with $\{x_1, x_2, \dots \} \subseteq \operatorname{supp}(h^{\star})$ such that for every time point $t \in \mathbbm{N}$ where $|\{x_1, \dots, x_t\}| = d$, there exists a $s \geq t$ such that $\Gcal(x_{1:s}) \notin \operatorname{supp}(h^{\star}) \setminus \{x_1, \dots, x_s\}.$ So, our candidate dimension should satisfy the property that when it is infinite, we can find arbitrarily large sequences of examples after which any generator is guaranteed to make a mistake. With this in mind, we are ready to present the Closure dimension, whose finiteness satisfies exactly this property.

\begin{definition}[Closure dimension] \label{def:gem} The \emph{Closure dimension} of $\Hcal$, denoted $\operatorname{C}(\Hcal)$, is the largest natural number $d \in \mathbbm{N}$ for which there exists \emph{distinct} $x_1, \dots, x_d \in \Xcal$ such that $\langle x_1, \dots, x_d \rangle_{\Hcal} \neq \bot$ and $|\langle x_1, \dots, x_d \rangle_{\Hcal}| < \infty.$  If this is true for arbitrarily large $d \in \mathbbm{N}$, then we say that $\operatorname{C}(\Hcal) = \infty.$ On the other hand, if this is not true for $d = 1$, we say that $\operatorname{C}(\Hcal) = 0.$
\end{definition}


\vspace{5pt}

The following lemma shows that the finiteness of $\operatorname{C}(\Hcal)$ is necessary for uniform generatability.  The high-level idea is that the adversary can force the learner to make a mistake at time point $t$, if there are no common positive examples amongst those hypotheses that contain $x_1, \dots, x_t$ in their support.  The finiteness of $\operatorname{C}(\Hcal)$ guarantees the existence of such a $t \in \mathbbm{N}$ and $x_1, \dots, x_t$.

\begin{lemma}[Necessity in Theorem \ref{thm:unifgen}] \label{lem:closnecc}
Let $\Xcal$ be countable and $\Hcal \subseteq \{0, 1\}^{\Xcal}$ be any  class satisfying the $\operatorname{UUS}$ property. If $\operatorname{C}(\Hcal) = \infty$, then $\Hcal$ is \emph{not} uniformly generatable. 
\end{lemma}

\begin{proof} Let $\Gcal$ be any generator and suppose $\operatorname{C}(\Hcal) = \infty$. We need to show that for every $d \in \mathbbm{N}$, there exists a $h^{\star} \in \Hcal$ and a sequence $(x_i)_{i \in \mathbbm{N}}$ with $\{x_1, x_2, \dots \} \subseteq \operatorname{supp}(h^{\star})$ such that for every time point $t \in \mathbbm{N}$ where $|\{x_1, \dots, x_t\}| = d$, there exists $s \geq t$ such that $\Gcal(x_{1:s}) \notin \operatorname{supp}(h^{\star}) \setminus \{x_1, \dots, x_s\}.$ 

To that end, fix a $d \in \mathbbm{N}$. Since $\operatorname{C}(\Hcal) = \infty$, we know that there exists some $d^{\star} \geq d$ and distinct $z_1, \dots, z_{d^{\star} }$ such that $\Hcal(z_1, \dots, z_{d^{\star}}) \neq \bot$ and $|\langle z_1, \dots, z_{d^{\star}}\rangle_{\Hcal}| < \infty.$  Since $\Hcal(z_{1:d^{\star}}) \subseteq \Hcal(z_{1:d})$, we also know that $|\langle z_1, \dots, z_d \rangle_{\Hcal}| < \infty.$

Let $p := |\langle z_1, \dots, z_d \rangle_{\Hcal}|.$ Note that for every $x \in \Xcal \setminus \langle z_1, \dots, z_d \rangle_{\Hcal}$, there exists a $h \in \Hcal(\langle z_1, \dots, z_d \rangle_{\Hcal})$ such that $x \notin \operatorname{supp}(h)$. Let $\hat{x}_{p} = \Gcal(\langle z_1, \dots, z_d \rangle_{\Hcal})$ denote the prediction of $\Gcal$ when given as input $\langle z_1, \dots, z_d \rangle_{\Hcal}$ sorted in its natural order. Without loss of generality suppose that $\hat{x}_{p} \notin \langle z_1, \dots, z_d \rangle_{\Hcal}$. Then, using the previous observation, there exists $h^{\star} \in \Hcal(\langle z_1, \dots, z_d \rangle_{\Hcal})$ such that $\hat{x}_{p} \notin \operatorname{supp}(h^{\star}) \setminus \langle z_1, \dots, z_d \rangle_{\Hcal}$.  Pick this $h^{\star}$ and consider the stream $x_1, x_2, \dots$ by sorting $\langle z_1, \dots, z_d \rangle_{\Hcal}$ in its natural ordering and then appending the stream $x_{p+1},x_{p+2},\dots \subseteq \operatorname{supp}(h^{\star})$ such that  $\langle z_1, \dots, z_d \rangle_{\Hcal} \cap \bigcup^{\infty}_{i = p+1} \{x_i\} =  \emptyset$. 

It remains to show that for every time point $t \in \mathbbm{N}$ where $|\{x_1, \dots, x_t\}| = d$, there exists $s \geq t$ such that $\Gcal(x_{1:s}) \notin \operatorname{supp}(h^{\star}) \setminus \{x_1, \dots, x_s\}.$  By definition, we know that $x_1, \dots, x_d$ are distinct. Thus, when $t = d$, we have that $|x_1, \dots, x_t| = d.$ Moreover, this is the only such time point. Accordingly, it suffices to show that there exists  $s \geq t$ such that $\Gcal(x_{1:s}) \notin \operatorname{supp}(h^{\star}) \setminus \{x_1, \dots, x_s\}.$  However, by definition, we know that  $\Gcal(x_{1:p}) \notin \operatorname{supp}(h^{\star}) \setminus \{x_1, \dots, x_{p}\}.$  Since $p \geq d$ and $d \in \mathbbm{N}$ was chosen arbitrarily, our proof is complete.
\end{proof}

The proof of Lemma \ref{lem:closnecc} also shows that the Closure dimension provides a quantitative lower bound on the optimal uniform generation sample complexity. Namely, for any class $\Hcal$ and generator $\Gcal$, we have that $d_{\Gcal} \geq \operatorname{C}(\Hcal).$ Next, we move to the sufficiency condition. Namely, the following lemma shows that the finiteness of $\operatorname{C}(\Hcal)$ is also sufficient for uniform generatability. The main idea is that if $\operatorname{C}(\Hcal) = d$, then one only needs to observe $d+1$ distinct examples before one can identify an infinite ``core" set $S \subseteq \Xcal$ that lies in the support of the hypothesis chosen by the adversary. The generator can then just play from the set $S$ for the remaining rounds. 

\begin{lemma}[Sufficiency in Theorem \ref{thm:unifgen}] \label{lem:clossuff}
Let $\Xcal$ be countable and $\Hcal \subseteq \{0, 1\}^{\Xcal}$ be any  class satisfying the $\operatorname{UUS}$ property. When $\operatorname{C}(\Hcal) < \infty$, there exists a generator $\Gcal$, such that for every $h \in \Hcal$ and any sequence $(x_i)_{i \in \mathbbm{N}}$ with $\{x_1, x_2, \dots \} \subseteq \operatorname{supp}(h)$, if there exists a $t \in \mathbbm{N}$ such that $|\{x_1, \dots, x_{t}\}| = \operatorname{C}(\Hcal) + 1$, then  $\Gcal(x_{1:s}) \in  \operatorname{supp}(h) \setminus \{x_1, \dots, x_s\}$ for all $s \geq t$.
\end{lemma}

\begin{proof} Let $0 \leq d < \infty$ and suppose $\operatorname{C}(\Hcal) = d$. Then,
for every distinct sequence of $d+1$ examples $x_1, \dots, x_{d+1}$ such that $\langle x_{1:d+1} \rangle_{\Hcal} \neq \bot$, we have that $|\langle x_1, \dots, x_{d+1} \rangle_{\Hcal}| = \infty.$ Consider the following generator $\Gcal$. Until $d+1$ distinct examples are observed, $\Gcal$ plays any $\hat{x}_s$. Suppose on round $t^{\star}$, $\Gcal$ observes $d+1$ distinct examples. Then, $\Gcal$ plays any $\hat{x}_s \in \langle x_1, \dots, x_{t^{\star}} \rangle_{\Hcal} \setminus \{x_1, \dots, x_s\}$ for all $s \geq t^{\star}$ . Let $h^{\star}$ be the hypothesis chosen by the adversary. It suffices to show that $\hat{x}_s \in \operatorname{supp}(h^{\star}) \setminus \{x_1, \dots, x_s\}$ for all $s \geq t^{\star}.$ However, this just follows from the fact that $|\langle x_1, \dots, x_{t^{\star}} \rangle_{\Hcal}| = \infty$ and $\langle x_1, \dots, x_{t^{\star}} \rangle_{\Hcal} \subseteq \operatorname{supp}(h^{\star})$. In particular, $|\langle x_1, \dots, x_{t^{\star}} \rangle_{\Hcal}| = \infty$ ensures that $\hat{x}_s$ is well-defined and $\langle x_1, \dots, x_{t^{\star}} \rangle_{\Hcal} \subseteq \operatorname{supp}(h^{\star})$ ensures that it always lies in $\operatorname{supp}(h^{\star}) \setminus \{x_1, \dots, x_s\}$.
\end{proof}

\begin{remark}
The generator in Lemma \ref{lem:clossuff} can be efficiently implemented given access to the following max-min oracle $\mathcal{O}_{\emph{max-min}}: 2^{\{0, 1\}^{\Xcal}} \times \Xcal^{\star} \rightarrow \Xcal$. Given a hypothesis class $\Hcal \subseteq \{0, 1\}^{\Xcal}$ and a finite sequence of examples $x_1, \dots, x_t$, $\mathcal{O}_{\emph{max-min}}$ returns 

$$\argmax_{x \in \Xcal\setminus\{x_1, \dots, x_t\}} \,  \min_{h \in \Hcal} \sum_{i=1}^{t} \mathbbm{1}\{h(x_i) \neq 1\} + \mathbbm{1}\{h(x) \neq 0\}.$$

The inner minimization is simply the output of an ERM oracle $\Ocal_{\emph{ERM}}$, whch given $\Hcal$ and the sample $S \in (\Xcal \times \{0, 1\})^{\star}$, outputs the minimal empirical loss on $S$ amongst all $h \in \Hcal$. Thus, the output of $\mathcal{O}_{\emph{max-min}}$ can be equivalently written as: 

$$\mathcal{O}_{\emph{max-min}}(\Hcal, x_{1:t}) = \argmax_{x \in \Xcal\setminus\{x_1, \dots, x_t\}}  \, \Ocal_{\emph{ERM}}(\Hcal, \{(x_1, 1), \dots, (x_t, 1), (x, 0)\}).$$

In fact, the generator in Lemma \ref{lem:clossuff} can be implemented with a single call to the max-min oracle on every round. To see why, suppose that $t \geq \operatorname{C}(\Hcal) + 1.$ Then, since $|\langle x_1, \dots, x_t \rangle_{\Hcal}| = \infty$, it must be the case that $|\langle x_1, \dots, x_t \rangle_{\Hcal} \setminus \{x_1, \dots, x_t\}| \neq \emptyset$ and for every $x \in \langle x_1, \dots, x_t \rangle_{\Hcal} \setminus \{x_1, \dots, x_t\}$, we have that 

$$\Ocal_{\emph{ERM}}(\Hcal, \{(x_1, 1), \dots, (x_t, 1), (x, 0)\}) \geq 1.$$
Accordingly,
$$\max_{x \in \Xcal\setminus\{x_1, \dots, x_t\}} \,  \Ocal_{\emph{ERM}}(\Hcal, \{(x_1, 1), \dots, (x_t, 1), (x, 0)\}) \geq 1$$
and therefore $\mathcal{O}_{\emph{max-min}}(\Hcal, x_{1:t})$ returns an element $\hat{x}_t \in \Xcal\setminus \{x_1, \dots, x_t\}$ such that:
$$\Ocal_{\emph{ERM}}(\Hcal, \{(x_1, 1), \dots, (x_t, 1), (\hat{x}_t, 0)\}) \geq 1 .$$
But, such an $\hat{x}_t$ must lie in $\langle x_1, \dots, x_t \rangle_{\Hcal}$, completing the proof. As a concluding remark, note that the max-min oracle should remind the reader of the min-max objective/two-player game used to motivate Generative Adversarial Networks (see Equation 1 in \cite{goodfellow2014generative}). In particular, for our min-max oracle, one can think of the minimizer/\emph{ERM} oracle as the discriminator and the outer maximizer as the generator. 
\end{remark}

Composing Lemmas \ref{lem:closnecc} and \ref{lem:clossuff} gives a characterization of uniform generatability.

\begin{theorem}[Characterization of Uniform Generatability] \label{thm:unifgen} Let $\Xcal$ be countable and $\Hcal \subseteq \{0, 1\}^{\Xcal}$ satisfy the $\operatorname{UUS}$ property. The following statements are equivalent. 

\begin{itemize}
\item[(i)] $\Hcal$ is uniformly generatable.
\item[(ii)] $\operatorname{C}(\Hcal) < \infty$.
\end{itemize}
\end{theorem}

We highlight that the Closure dimension not only provides a qualitative characterization of uniform generatability, but also a quantitative characterization -- the optimal uniform generation sample complexity is exactly $\Theta(\operatorname{C}(\Hcal))$. \cite{kleinberg2024language} proved that all countable classes are generatable in the limit. Are there uncountable classes that are generatable in the limit? What about uniformly generatable? Our next result proves the existence of \emph{uncountably} infinite hypothesis classes that are uniformly generatable. This provides an  improvement over \cite{kleinberg2024language} in two ways. First, it shows that finiteness is not necessary for uniformly generatability. Second, it shows that the countableness of $\Hcal$ is not necessary for generatability in the limit.

\begin{lemma}\label{lem:uncountunifgen} Let $\Xcal$ be countable. There exists a class $\Hcal \subseteq \{0, 1\}^{\Xcal}$  that is uncountably large,  satisfies the $\operatorname{UUS}$ property, and is uniformly generatable. 
\end{lemma}

\begin{proof} Let $\Xcal = \mathbbm{Z}$ and $\Hcal = \{x \mapsto \mathbbm{1}\{x \in A \text{ or } x \leq 0\}: A \in  2^\mathbbm{N}\}$. It is not hard to see that $\Hcal$ satisfies the UUS property. Moreover, since $2^{\mathbbm{N}}$ is uncountably large, so is $\Hcal$. Finally, to see that $\Hcal$ is uniformly generatable, note that for every $x \in \mathbbm{Z}$, we have that  $\langle x \rangle_{\Hcal} = \mathbbm{Z}_{\leq 0}$. Thus, $\operatorname{C}(\Hcal) = 0$ and $\Hcal$ is trivially uniformly generatable. 
\end{proof}



\subsection{Non-uniform Generatability}

We next move to characterize non-uniform generatability. Similar to non-uniform PAC and online learnability, we show that our characterization of uniform generatability leads to a characterization of non-uniform generatability.

\begin{theorem}[Characterization of Non-uniform Generatability] \label{thm:nonunifgen} Let $\Xcal$ be countable and $\Hcal \subseteq \{0, 1\}^{\Xcal}$ satisfy the $\operatorname{UUS}$ property. The following statements are equivalent. 

\begin{itemize}
\item[(i)] $\Hcal$ is non-uniformly generatable.
\item[(ii)] There exists a \emph{non-decreasing} sequence of classes $\Hcal_1 \subseteq \Hcal_2 \subseteq \dots$ such that $\Hcal =\bigcup_{i\in\naturals} \Hcal_i$ and $\operatorname{C}(\Hcal_n)<\infty$ for every $n\in\naturals.$
\end{itemize}
\end{theorem}

\begin{remark}
    One might wonder whether we can drop ``non-decreasing" in statement (ii) of Theorem \ref{thm:nonunifgen} and write $\Hcal$ as the countable union of uniformly generatable classes $\Hcal_1, \Hcal_2, \dots$. However, this is provably \emph{false} as we will show in Lemma \ref{lem:notnonunifgen} -- while such a condition is necessary, it is \emph{not} sufficient. This differs from non-uniform PAC/online learnability whose characterization \emph{can} be written in terms of the countable union of (uniformly) PAC/online learnable classes \citep{ShwartzDavid, lu2023non}. This is one of the many differences between generation and prediction which we expand more on in Section \ref{sec:genvpred}.
\end{remark}

\cite{kleinberg2024language} showed that every countable hypothesis class is generatable in the limit. However, we can use Theorem \ref{thm:nonunifgen} to show that every countable hypothesis class is actually non-uniformly generatable.  This is a stronger result as non-uniform generation implies generation in the limit. We note that recent and independent work by \cite{charikar2024exploring} (see Theorem 1) also establish Corollary \ref{cor:countnonunif}.

\begin{corollary}[Countable Classes are Non-uniformly Generatable] \label{cor:countnonunif}
Let $\Xcal$ be countable and $\Hcal\subseteq \{0,1\}^{\Xcal}$ be any hypothesis class that satisfies the \emph{UUS} property. If $\Hcal$ is countable, then $\Hcal$ is non-uniformly generatable.
\end{corollary}
\begin{proof} (of Corollary \ref{cor:countnonunif})
    Suppose $\Hcal$ is a countable hypothesis class satisfying the UUS property. Consider an arbitrary enumeration $h_1, h_2, \dots$ of $\Hcal$. Let $\Hcal_n=\{h_1, \dots, h_n\}$ for all $ n\in\naturals$. Then, $\Hcal_1, \Hcal_2, \dots$ is a non-decreasing sequence of classes such that $\Hcal = \bigcup_{n \in \naturals} \Hcal_n.$ Moreover, since for every $n \in \naturals$, we have that $|\Hcal_{n}| = n < \infty$, Theorem \ref{thm:genfinite} gives that $\Hcal_n$ is uniformly generatable, completing the proof. 
\end{proof}

We now prove Theorem \ref{thm:nonunifgen} across two lemmas, starting with the necessity direction.

\begin{lemma}[Necessity in Theorem \ref{thm:nonunifgen}]
Let $\Xcal$ be countable and $\Hcal \subseteq \{0, 1\}^{\Xcal}$ be any  class satisfying the $\operatorname{UUS}$ property. If $\Hcal$  is non-uniformly generatable, then there exists a sequence of \emph{non-decreasing, uniformly generatable} classes $\Hcal_1 \subseteq \Hcal_2 \subseteq \cdots$ such that $\Hcal =\bigcup_{n\in\naturals} \Hcal_n$.
\end{lemma}

\begin{proof} 
Suppose $\Hcal$ is non-uniformly generatable and $\Gcal$ is a non-uniform generator for $\Hcal$. For every $h\in\Hcal$, let $d_h\in\naturals$ be the smallest natural number such that for any sequence $x_1, x_2, \dots$ with $\{x_1,x_2,\dots\}\subset \operatorname{supp}(h)$, if there exists a $t\in\naturals$ such that $|\{x_1,\dots,x_t\}|=d_h$, then $\Gcal(x_{1:s})\in\operatorname{supp}(h)\setminus\{x_1,\dots,x_s\}$ for all $s\ge t$. Let $\Hcal_n:=\{h\in\Hcal:d_h \leq n\}$ for all $n\in\naturals$. Then, by the definition of $\Gcal$, we know for every $n\in\naturals$, $\Gcal$ is a uniform generator for $\Hcal_n$, and therefore $\Hcal_n$ is uniformly generatable.  The proof is complete after noting that $\Hcal_1 \subseteq \Hcal_2 \subseteq \cdots$ and $\Hcal=\bigcup_{n\in\naturals}\Hcal_n$.
\end{proof}

The next lemma shows that the condition in Theorem \ref{thm:nonunifgen} is sufficient. Our proof is constructive and by a reduction -- given uniform generators $\Gcal_1, \Gcal_2, \dots$ for $\Hcal_1, \Hcal_2, \dots$ and their uniform generation sample complexities $d_{\Gcal_1}, d_{\Gcal_2}, \dots$ \footnote{One actually only needs an upper bound on the uniform generation sample complexities.}, we construct a non-uniform generator $\Gcal$ for $\Hcal = \bigcup_{n \in \naturals} \Hcal_n.$ This aligns with existing sufficiency proofs for non-uniform PAC and online learning \citep{ShwartzDavid, lu2023non}, which are also through reductions.

\begin{lemma}[Sufficiency in Theorem \ref{thm:nonunifgen}]
Let $\Xcal$ be countable and $\Hcal \subseteq \{0, 1\}^{\Xcal}$ be any class satisfying the $\operatorname{UUS}$ property. If there exists a sequence of \emph{non-decreasing, uniformly generatable}  classes $\Hcal_1 \subseteq \Hcal_2 \subseteq \dots$ such that $\Hcal = \bigcup_{n \in \naturals} \Hcal_n$, then $\Hcal$ is non-uniformly generatable. 
\end{lemma}


\begin{proof}
Suppose $\Hcal \subseteq \{0, 1\}^{\Xcal}$ is a class satisfying the $\operatorname{UUS}$ property for which there exists a sequence of non-decreasing, uniformly generatable classes $\Hcal_1 \subseteq \Hcal_2 \subseteq \dots$ with  $\Hcal = \bigcup_{n = 1}^{\infty} \Hcal_n$. By definition of uniform generatability, for every $n \in \naturals$, there exists a uniform generator $\Gcal_n$ for $\Hcal_n$. Let $d_{\Gcal_n}$ denote the uniform generation sample complexity of $\Gcal_n$ with respect to $\Hcal_{n}$. We consider two cases: $\sup_{n\in\naturals}d_{\Gcal_n}=\infty$ and $\sup_{n\in\naturals}d_{\Gcal_n}<\infty$.

In the first case, consider the following generator $\Gcal$. Fix $t \in \mathbbm{N}$ and consider any sequence $\{x_1,\dots,x_t\}$ such that $|\Hcal(x_1, \dots, x_t)| \geq 1$. Let  $d_t := |\{x_1,\dots,x_t\}|$ be the number of unique examples. $\Gcal$ first computes 
$$n_t = \max\{n \in \mathbbm{N} : d_{\Gcal_n} \leq d_t\}\cup\{0\}.$$

If $n_t = 0$,   $\Gcal$ plays any $\hat{x}_t \in \Xcal$. If $n_t > 0$, $\Gcal$ uses $\Gcal_{n_t}$ to generate new instances, which means 
$$\Gcal(x_{1},\dots,x_t)=\Gcal_{n_t}(x_1,\dots,x_t).$$


We now prove that such a $\Gcal$ is a non-uniform generator for $\Hcal$. To that end, 
let $h^{\star}$ be the hypothesis chosen by the adversary and suppose that $h^{\star}$ belongs to $\Hcal_{n^{\star}}$. We show that for every $t \in \naturals$ and $\{x_1,\dots,x_t\} \subseteq \operatorname{supp}(h^{\star})$ such that $d_t := |\{x_1,\dots,x_t\}| \geq d_{\Gcal_{n^{\star}}}$,  we have $\Gcal(x_{1:t}) \in \operatorname{supp}(h^{\star})\setminus\{x_1,\dots,x_t\}$. By definition, $\Gcal$ first computes 
$$n_t = \max\{n \in \mathbbm{N} : d_{\Gcal_n} \leq d_t\} \cup \{0\}.$$

Note that $n_t\ge n^{\star}$ since $d_{\Gcal_{n^{\star}}} \leq d_t$. Thus, $|\Hcal_{n_t}(x_{1:t})| \geq 1$ since $h^{\star} \in \Hcal_{n_t}$. Accordingly, by construction of $\Gcal$, it uses $\Gcal_{n_t}$ to generate a new instance. The proof is complete by noting that $h^{\star} \in \Hcal_{n_t}$ and $d_t \geq d_{\Gcal_{n_t}}$ which guarantees that $\Gcal_{n_t}(x_1,\dots, x_t)\in \operatorname{supp}(h^{\star})\setminus\{x_1,\dots,x_t\}$.

In the second case, suppose $\sup_{n\in\naturals}d_{\Gcal_n}:=c<\infty$. Consider the following generator $\Gcal$. Fix $t \in \mathbbm{N}$ and consider any sequence $\{x_1,\dots,x_t\}$ such that $|\Hcal(x_1, \dots, x_t)| \geq 1$. Let  $d_t := |\{x_1,\dots,x_t\}|$ be the number of unique examples. If $d_t<c$, $\Gcal$ plays any $\hat{x}_t\in\Xcal$. Otherwise, if $d_t \ge c$, $\Gcal$ uses $\Gcal_t$ to generate a new sample, which means
$$\Gcal(x_{1},\dots,x_t)=\Gcal_{t}(x_1,\dots,x_t).$$

Let $h^{\star}$ be the hypothesis chosen by the adversary and suppose $h^{\star}$ belongs to $\Hcal_{n^{\star}}$. We show that for every $t \in \naturals$ and $\{x_1,\dots,x_t\} \subseteq \operatorname{supp}(h^{\star})$ such that $d_t :=|\{x_1,\dots,x_t\}|> \max(c,n^{\star})$,  we have $\Gcal(x_{1:t}) \in \operatorname{supp}(h^{\star})\setminus\{x_1,\dots,x_t\}$.  Because $t\ge d_t \geq n^{\star}$, we have that $h^{\star} \in \Hcal_{t}$. Therefore, by construction, $\Gcal$ generates a new sample according to $\Gcal_t$. The proof is complete after noting that $t>\max(c,n^{\star})$ thus $\Gcal_{t}(x_1,\dots, x_t)\in \operatorname{supp}(h^{\star})\setminus\{x_1,\dots,x_t\}$.
\end{proof}

\begin{remark}
Since only upper bounds on the uniform generation sample complexities of $\Gcal_1$, $\Gcal_2, \dots$ are needed in the proof of Lemma \ref{thm:nonunifgen}, the algorithm in the proof of Lemma \ref{thm:nonunifgen} can be efficiently implemented as long as each $\Gcal_i$ is efficient. This is because in each round $t \in \naturals$, the number $n_t$ can be efficiently computed if the sample complexities $d_{\Gcal_1}, d_{\Gcal_2}, \dots$ are non-decreasing. However, even if the  sample complexities $d_{\Gcal_1}, d_{\Gcal_2}, \dots$ are not presented in non-decreasing order, we can run the algorithm on a new sequence of sample complexities $d^{\prime}_{\Gcal_{1}}, d^{\prime}_{\Gcal_{2}}, \dots$ such that $d^{\prime}_{\Gcal_{1}} = d_{\Gcal_{1}}$ and $d^{\prime}_{\Gcal_{i}} = \max\{d^{\prime}_{\Gcal_{i-1}}, d_{\Gcal_{i}}\}$ for all $i \geq 2.$
\end{remark}

We end this section by proving that uniform generatability is strictly harder than non-uniform generatability, completing part (i) of Proposition \ref{prop:gencomp}.

\begin{lemma}[Uniform Generatability $\neq$ Non-uniform Generatability ] \label{lem:nonunifvsunifgen}  Let $\Xcal$ be countable. There exists a countable class $\Hcal \subseteq \{0, 1\}^{\Xcal}$ that satisfies the $\operatorname{UUS}$ property, is non-uniformly generatable, but not uniformly generatable. 
\end{lemma}

\begin{proof} 


Let $\Xcal = \mathbbm{Z}$. Let $E$ denote the set of all even negative integers and $O$ the set of all odd negative integers. Consider the hypothesis classes 
\begin{equation} \label{eq:He}
\Hcal^{e} =  \Biggl\{x \mapsto \mathbbm{1}\Biggl\{x \in \Bigl\{\frac{d(d-1)}{2} + 1, \dots, \frac{d(d-1)}{2} + d\Bigl\} \text{ or } x \in E\Biggl\}: d \in \mathbbm{N}\Biggl\}
\end{equation}

and 
\begin{equation} \label{eq:Ho}
\Hcal^{o} = \Biggl\{x \mapsto \mathbbm{1}\Biggl\{x \in \Bigl\{\frac{d(d-1)}{2} + 1, \dots, \frac{d(d-1)}{2} + d\Bigl\} \text{ or } x \in O\Biggl\}: d \in \mathbbm{N}\Biggl\}
\end{equation}

and define $\Hcal = \Hcal^{e} \cup \Hcal^{o}$. First, its not too hard to see that $\Hcal$ satisfies the $\operatorname{UUS}$. Second, we claim that $\operatorname{C}(\Hcal) = \infty$. To see why, we need to show that for every $d \in \mathbbm{N}$, there exists a $d^{\star} \geq d$ and a distinct sequence of examples $x_1, \dots, x_{d^{\star}}$, such that $|\Hcal(x_1, \dots, x_{d^{\star}})| \geq 1$ and 

$$|\langle x_1, \dots, x_{d^{\star}} \rangle_{\Hcal}| < \infty.$$

To that end, pick any $d \in \mathbbm{N}$ and let $d^{\star} = d$. Consider the sequence of examples $x_1 = \frac{d(d-1)}{2} + 1, \dots, x_d = \frac{d(d-1)}{2} + d$. First, observe that this is a sequence of $d = d^{\star}$ distinct examples. Then,  
$$\langle x_1, \dots, x_{d} \rangle_{\Hcal} = \operatorname{supp}(h^e_d) \cap \operatorname{supp}(h^o_d) = \{x_1, \dots, x_d\}$$
\noindent where we let 
$$h^e_d := \mathbbm{1}\Bigg\{x \in \Bigl\{\frac{d(d-1)}{2} + 1,  \dots, \frac{d(d-1)}{2} + d\Bigl\} \text{ or } x \in E\Biggl\},$$
and 
$$h^o_d := \mathbbm{1}\Bigg\{x \in \Bigl\{\frac{d(d-1)}{2} + 1,  \dots, \frac{d(d-1)}{2} + d\Bigl\} \text{ or } x \in O\Biggl\}.$$
Thus, we have that 
$$|\langle x_1, \dots, x_{d} \rangle_{\Hcal}| < \infty.$$
\noindent Since $d$ was chosen arbitrarily, this is true for all $d \in \mathbbm{N}$, implying that $\operatorname{C}(\Hcal) = \infty.$ Thus, by Theorem \ref{thm:unifgen}, we have that $\Hcal$ is not uniformly generatable. To show that $\Hcal$ is non-uniformly generatable, note that $\Hcal$ is countable. Thus, Corollary \ref{cor:countnonunif} completes the proof. 
\end{proof}

\subsection{Generatability in the Limit}

\cite{kleinberg2024language} showed that all countable classes are generatable in the limit. Here, we provide an alternate sufficiency condition for generatability in the limit which, in conjunction with countableness, significantly expands the collection of classes which are generatable in the limit.

\begin{theorem} [Sufficient Condition for Generatability in the Limit] \label{thm:geninlim}  Let $\Xcal$ be countable and $\Hcal \subseteq \{0, 1\}^{\Xcal}$ be any class satisfying the \emph{UUS} property. If there exists a finite sequence of classes $\Hcal_1, \Hcal_2, \dots, \Hcal_n$ such that $\Hcal = \bigcup_{i = 1}^n \Hcal_i$ and $\operatorname{C}(\Hcal_i) < \infty$ for all $i \in [n]$, then $\Hcal$ is generatable in the limit. 
\end{theorem}

\begin{proof} Let $\Hcal = \bigcup_{i = 1}^n \Hcal_i$ be such that $\operatorname{C}(\Hcal_i) < \infty$ for all $i \in [n].$ Let $c := \max_{i \in [n]} \operatorname{C}(\Hcal_i).$ Consider the following generator $\Gcal$.  Let $t^{\star} \in \mathbbm{N}$ be the smallest time point for which $|\{x_1, \dots, x_{t^{\star}}\}| = c+1.$ $\Gcal$ plays arbitrarily up to, but not including, time point $t^{\star}$. On time point $t^{\star}$, $\Gcal$ computes $\langle x_1, \dots, x_{t^{\star}} \rangle_{\Hcal_i}$ for all $i \in [n]$. Let $S \subseteq [n]$ be the subset of indices such that $i \in S$ if and only if $\langle x_1, \dots, x_{t^{\star}} \rangle_{\Hcal_i} \neq \bot.$ For every $i \in S$, let $(z^{(i)}_j)_{j \in \mathbbm{N}}$ be the natural ordering of  $\langle x_1, \dots, x_{t^{\star}} \rangle_{\Hcal_i}$, which is guaranteed to exist since $\Xcal$ is countable. For every $t \geq t^{\star}$, sequence of revealed examples $x_1, \dots, x_t$, and $i \in S$, $\Gcal$ computes 

\begin{equation} \label{eq:geninlim}
n_t^i := \max\{n \in \mathbbm{N} : \{z^{(i)}_{1}, \dots, z^{(i)}_{n}\}  \subset \{x_1, \dots, x_t\}\}
\end{equation}

\noindent and $i_t\in \argmax_{i \in S} n_t^i.$ Finally, $\Gcal$ plays any $\hat{x}_t \in \langle x_1, \dots, x_{t^{\star}} \rangle_{\Hcal_{i_t}} \setminus \{x_1, \dots, x_t\}.$ We claim that $\Gcal$ generates from $\Hcal$ in the limit. 

Let $h^{\star} \in \Hcal$ be the hypothesis chosen by the adversary  and $x_1,x_2, \dots$ be the selected enumeration of $\operatorname{supp}(h^{\star})$. Let $c =\max_{i \in [n]} \operatorname{C}(\Hcal_i)$ and  $t^{\star} \in \mathbbm{N}$ be the smallest time point for which $|\{x_1, \dots, x_{t^{\star}}\}| = c+1.$ By definition of $\operatorname{C}(\cdot)$, we know that for every $j \in S$, $|\langle x_1, \dots, x_{t^{\star}} \rangle_{\Hcal_j}| = \infty.$ Let $S^{\star} \subseteq S$ be such that $i \in S^{\star}$ if and only if $\langle x_1, \dots, x_{t^{\star}}\rangle_{\Hcal_i} \subseteq \operatorname{supp}(h^{\star}).$ 
It suffices to show that there exists a finite time point $s^{\star} \in \mathbbm{N}$ such that for all $t \geq s^{\star}$, we have that $i_t \in S^{\star}.$ To see why such an $s^{\star}$ must exist, pick some $j^{\star} \in S^{\star}.$ Note that $n_t^{j^{\star}} \rightarrow \infty$ because $x_1, x_2, \dots$ is an enumeration of $\operatorname{supp}(h^{\star})$. On the other hand, observe that for every $j \notin S^{\star}$, there exists a $n^j \in \mathbbm{N}$ such that $n_t^j \leq n^j$. This is because, if $j \notin S^{\star}$, then there must be an index $n^j \in \mathbbm{N}$ such that $z^{(j)}_{n^j} \notin \operatorname{supp}(h^{\star}).$ Thus, $n_t^j$ computed in Equation \ref{eq:geninlim} must be at most $n^j$. Since there are at most a finite number of indices not in $S^{\star}$, we have that $\max_{j \notin S^{\star}} n^j < \infty$, which means that eventually, $n_t^{j^{\star}} > n_t^j$ for all $j \notin S^{\star}$, and thus there exists a $s^{\star} \in \mathbbm{N}$ such that $i_t \in S^{\star}$ for all $t \geq s^{\star}.$ This completes the proof. \end{proof}

\begin{remark}
    The algorithm in the proof of Theorem \ref{thm:geninlim} can be efficiently implemented given access to an \emph{ERM} oracle $\Ocal: 2^{\{0, 1\}^{\Xcal}} \times (\Xcal \times \{0, 1\})^{\star} \rightarrow \naturals \cup \{0\}$ and an oracle $\Ocal_{\text{C}}: \Hcal \rightarrow \naturals$ that can compute the Closure dimension. In particular, before the game begins, the algorithm uses $\Ocal_{\text{C}}$ to compute Closure dimension for each class. After this, the algorithm only needs to use finite calls to an \emph{ERM} oracle in each round $t \in \naturals$. To see why, first note that the algorithm plays arbitrarily until time point $t^{\star}$ where $c+1$ distinct examples are revealed, where $c =\max_{i \in [n]} \operatorname{C}(\Hcal_i)$ (which can be computed from $n$ calls to $\Ocal_{\text{C}}$ at the start.) On round $t^{\star}$, the subset of indices $S \subseteq [n]$ for which the closure is not $\bot$ can be computed using at most $n$ calls to $\Ocal_{\emph{ERM}}$ by the discussion in Remark \ref{rem:erm}. Next, note that for every $t \geq t^{\star}$, and every $i \in S$, the number $n_t^i$ can be computed using on finite number of calls to a closure membership oracle (see Remark \ref{rem:erm}). Since closure membership can be computed using a single call to $\Ocal_{\emph{ERM}}$, the computation of $n_t^i$ for all $i \in S$ can be computed using only a finite number of calls to $\Ocal_{\emph{ERM}}$. Finally, because $\hat{x}_t$ only needs a finite number of closure membership calls to compute, it can also be computed using a finite number of calls to $\Ocal_{\emph{ERM}}$. 
\end{remark}

The following lemma gives examples of uncountably infinite classes that, using Theorem \ref{thm:geninlim}, are generatable in the limit.

 \begin{corollary} Let $\Xcal = \mathbbm{N}$ and $S_1, \dots, S_n \subseteq \mathbbm{N}$ be any finite sequence of countable infinite subsets of $\mathbb{N}$. For every $i \in [n]$, let $\Hcal_i = \{x \mapsto \mathbbm{1}\{x \in S_i \cup A\}: A \in 2^{\mathbbm{N}}\}.$ Then, $\Hcal = \bigcup_{i=1}^n \Hcal_i$ is generatable in the limit. 
 \end{corollary} 

 \begin{proof} Note that $\operatorname{C}(\Hcal_i) = 0$ for all $i \in [n].$ Thus, Theorem \ref{thm:geninlim} gives that $\Hcal$ is generatable in the limit. 
 \end{proof}

 Note that the condition in Theorem \ref{thm:geninlim} is not sufficient for non-uniform generatability, as evidenced by the example in Lemma \ref{lem:notnonunifgen}. Indeed, in this example we give two uniformly generatable classes whose union is no longer non-uniformly generatable. One might ask whether the sufficiency condition in Theorem \ref{thm:geninlim} can be extended to account for classes $\Hcal$ which can be written as the \emph{countable} union of uniformly generatable classes. Lemma \ref{lem:hardgeninlim} shows that this is actually \emph{not} the case -- there exists a countable sequence of uniformly generatable classes whose union is not generatable in the limit! This result showcases the hardness of characterizing generatability in the limit. 
 
Nevertheless, in Appendix \ref{app:weaksuff}, we do give an even weaker sufficiency condition for generatability in the limit. This weaker sufficiency conditon is written in terms of a new property of a hypothesis class called the ``Eventually Unbounded Closure (EUC)" property. Informally, a class $\Hcal$ satisfies the EUC property if, like the name suggest, for any valid sequence of positive examples, there exists a finite time point $t \in \naturals$ at which the closure of the sequence $x_1, \dots, x_t$ with respect to $\Hcal$ becomes infinite. Our weakening of the sufficiency condition in Theorem \ref{thm:weaksuff} replaces uniform generatability with the EUC property. 
 
It turns out that for countable classes, the EUC property is stronger than non-uniform generatability (see Appendix \ref{app:weaksuff}). Thus, we leave as an open question whether uniform generatability in Theorem \ref{thm:weaksuff} can be replaced by non-uniform generatability, and more generally, what is the complete characterization of generatability in the limit.

 We end this section by showing that there are classes which are generatable in the limit but not non-uniformly generatability. This, along with Lemma \ref{lem:nonunifvsunifgen}, completes the proof of Proposition \ref{prop:gencomp}.

\begin{lemma}[Non-uniform Generatability $\neq$ Generatability in the Limit] \label{lem:notnonunifgen}
    Let $\Xcal$ be countable. There exists a class $\Hcal\subseteq\{0,1\}^\Xcal$ which satisfies the \emph{USS} property that is generatable in the limit but not non-uniformly generatable.
\end{lemma}

\begin{proof}
    Let $\Xcal=\integers$ and $\Hcal=\{x \mapsto \mathbbm{1}\{x\in A\text{ or } x\leq 0\}:A \in 2^\naturals\}\cup\{x\mapsto\mathbbm{1}\{x\in\naturals\}\}$. Observe that $\Hcal$ satisfies the UUS property. We first show that $\Hcal$ is generatable in the limit. Let $\Gcal$ be a generator such that for any valid sequence $\{x_1,\dots, x_t\}$, if $x_1, \dots, x_t$ are all positive, then $\Gcal$ plays any $\hat{x}_t\in \naturals\setminus\{x_1,\dots, x_t\}$. Otherwise, it plays any $\hat{x}_t\in \integers_{\leq 0}\setminus\{x_1,\dots, x_t\}$. Now, suppose the adversary picks a $h^{\star} \in \Hcal$ and an enumeration  $x_1, x_2,\dots$ of $\operatorname{supp}(h^{\star})$. If $\operatorname{supp}(h^{\star}) = \naturals$, then by our construction of $\Gcal$, we have that $\Gcal(x_{1:s}) \in \operatorname{supp}(h^{\star}) \setminus \{x_1,\dots,x_s\}$ for all $s\ge 1$. Otherwise if $\integers_{\leq 0} \subseteq \operatorname{supp}(h^{\star})$, since $\bigcup_{i\in\naturals}\{x_i\}=\operatorname{supp}(h^{\star})$, there exists a $t^{\star}\in\naturals$ such that $x_{t^{\star}}\leq 0$. Then, $\Gcal(x_{1:s}) \in \integers_{\leq 0} \setminus \{x_1,\dots, x_s\} \subseteq \operatorname{supp}(h^{\star}) \setminus \{x_1,\dots,x_s\}$ for any $s\ge t^{\star}$. Thus, $\Gcal$ is a valid generator and $\Hcal$ is generatable in the limit.

    Now, we will show that $\Hcal$ is not non-uniformly generatable. Suppose for the sake of contradiction that $\Hcal$ is non-uniformly generatable. \ambuj{i'm really curious whether this can be shown using our char for non-unif gen. did you try? is there is a roadblock?} Then, there exists a non-uniform generator $\Gcal$ for $\Hcal$. For every $h \in \Hcal$,  let $d_h\in\naturals$ be the smallest natural number such that for any sequence $(x_i)_{i\in\naturals}$ with $\{x_1,x_2,\dots\}\subseteq \operatorname{supp}(h)$, if there exists a $t\in\naturals$ such that $|\{x_1,\dots,x_t\}|=d_h$, then $\Gcal(x_{1:s})\in\operatorname{supp}(h)\setminus\{x_1,\dots,x_s\}$ for all $s\ge t$. We now construct a $h \in \Hcal$ such that $d_h \ge n, \forall n\in\naturals$, which leads to a contradiction.

    Let $h_0\in\Hcal$ and $\operatorname{supp}(h_0)=\naturals$, then for any observed sequence $\{x_1,\dots,x_t\}\subset\naturals = \operatorname{supp}(h_0)$ such that $|\{x_1,\dots,x_t\}| \ge d_{h_0}$, we have that $\Gcal(x_{1:t})\in\naturals \setminus \{x_1, \dots, x_t\}$. Now let $\{p_n\}_{n\in\naturals}=\{2,3,5,7,\dots\}$ be the set of all prime numbers. Let $h_1 \in \Hcal$ be such that $\operatorname{supp}(h_1)=\{p_n\}_{n\in\naturals}\cup \integers_{\le 0}$. Let $d_1=\max(d_{h_1},d_{h_0})$, then by definition, we have that 
    $$\Gcal(\{2,\dots,p_{d_1}\})\in (\naturals \cap \{p_{n}\}_{n\in\naturals})\setminus\{2,\dots,p_{d_1}\} = \{p_{n}\}_{n\in\naturals} \setminus\{2,\dots,p_{d_1}\}.$$
    Let $h_2 \in \Hcal$ be such that 
    $$\operatorname{supp}(h_2)=\{2,\dots,p_{d_1},p_{d_1+1}^2,p_{d_1+2}^2,\dots\} \cup \integers_{\le 0}.$$
    Denote $d_2=d_{h_2}$, then $d_2\ge d_1+1$ since $\Gcal(\{2,\dots,p_{d_1}\}) \in (\naturals \cap \operatorname{supp}(h_1)) \setminus \{2,\dots,p_{d_1}\}$, which means that,
    $$\Gcal(\{2,\dots,p_{d_1}\})\notin \operatorname{supp}(h_2)\setminus \{2,\dots,p_{d_1}\}.$$
    Let $h_3 \in \Hcal$ such that 
    $$\operatorname{supp}(h_3)=\{2,\dots,p_{d_1},p_{d_1+1}^2,\dots, p_{d_2}^2,p_{d_2+1}^3,p_{d_2+2}^3,\dots\}\cup \integers_{\le 0}.$$
    Denote $d_3=d_{h_3}.$ Suppose, we observe
    $$\{x_{1},\dots,x_{d_2}\}=\{2,\dots,p_{d_1},p^2_{d_1+1},\dots,p^2_{d_2}\}.$$
    Then, it must be the case that $\Gcal(x_{1:d_2})\in \naturals \cap \operatorname{supp}(h_2)\setminus \{x_1,\dots,x_{d_2}\}$. Since 
    $$\left(\naturals \cap \operatorname{supp}(h_2)\setminus\{x_1,\dots, x_{d_2}\}\right)\cap \operatorname{supp}(h_3)=\emptyset, $$
    we have $\Gcal(x_{1:d_2})\notin\operatorname{supp}(h_3) \setminus \{x_{1:d_2}\}$ and as a result $d_3\ge d_2+1$.
    Inductively, suppose $h_1,h_2,\dots,h_n$ and $d_1,d_2,\dots, d_n$ are all defined. Let $h_{n+1} \in \Hcal$ be such that
    
    $$\operatorname{supp}(h_{n+1})=\{2,\dots,p_{d_1},\dots\dots, p_{d_{n-1}+1}^{n},\dots, p_{d_n}^{n},p_{d_n+1}^{n+1}, p_{d_n+2}^{n+1},\dots\dots\}\cup \integers_{\le 0}.$$
    Let $d_{n+1}=d_{h_{n+1}}$. Then $d_{n+1}\ge d_n+1$ since
    $$\Gcal(\{2,\dots,p_{d_1},\dots\dots, p_{d_{n-1}+1}^{n},\dots, p_{d_n}^{n}\})\notin \operatorname{supp}(h_{n+1}).$$
    Finally, let $h_{\infty}\in\Hcal$ be such that

    $$\operatorname{supp}(h_{\infty})=\{2,\dots,p_{d_1},p_{d_1 + 1}^{2},\dots,p_{d_2}^{2},p_{d_2+1}^{3}, \dots,  p_{d_3}^{3}, p_{d_3 + 1}^{4}, \dots\dots\}\cup \integers_{\le 0}.$$
    For every $t \in \naturals$, consider the sequence 
    $$\{x_1,\dots,x_{d_t}\}=\{2,\dots,p_{d_1}, p_{d_1+1}^2,\dots, p_{d_2}^{2},p_{d_2+1}^{3},\dots,p_{d_{t-1}+1}^{t},\dots, p_{d_t}^{t} \}.$$
    Then, $\Gcal(x_{1:d_t})\in \naturals \cap \operatorname{supp}(h_{t})\setminus\{x_1,\dots,x_{d_t}\}.$ Since 
    $$\left(\naturals \cap \operatorname{supp}(h_t)\setminus\{x_1,\dots,x_{d_t}\}\right)\cap (\operatorname{supp}(h_{\infty}) \setminus\{x_1,\dots,x_{d_t}\})=\emptyset,$$ it must be the case that $d_{h_{\infty}}\ge d_t,\forall t\in\naturals$. Since $d_t\to\infty$ as $t\to\infty$, this leads to a contradiction, as we have found a hypothesis $h_{\infty} \in \Hcal$ for which there is no uniform upper bound $d \in \mathbbm{N}$ such  that $\Gcal$ perfectly generates new examples after observing $d$ unique examples. 
\end{proof}







    




\section{Generation is Unlike Prediction} \label{sec:genvpred}

In this section, we flesh out the landscape of generation versus prediction for countable hypothesis classes. Namely, we seek to compare generation with prediction and understand how these two properties of hypothesis classes compare with one another. To evaluate the predictability of a hypothesis class $\Hcal$, we use the standard notions of PAC and online learnability. These two properties have been studied by learning theorists for decades, culminating in a precise quantitative characterization of which classes are PAC and online learnable. In particular, it is well known that the VC dimension and Littlestone dimension characterize PAC and online learnability respectively \citep{vapnik1971uniform, Littlestone1987LearningQW}. 

\begin{figure}
    \centering
\includegraphics[width=0.4\linewidth]{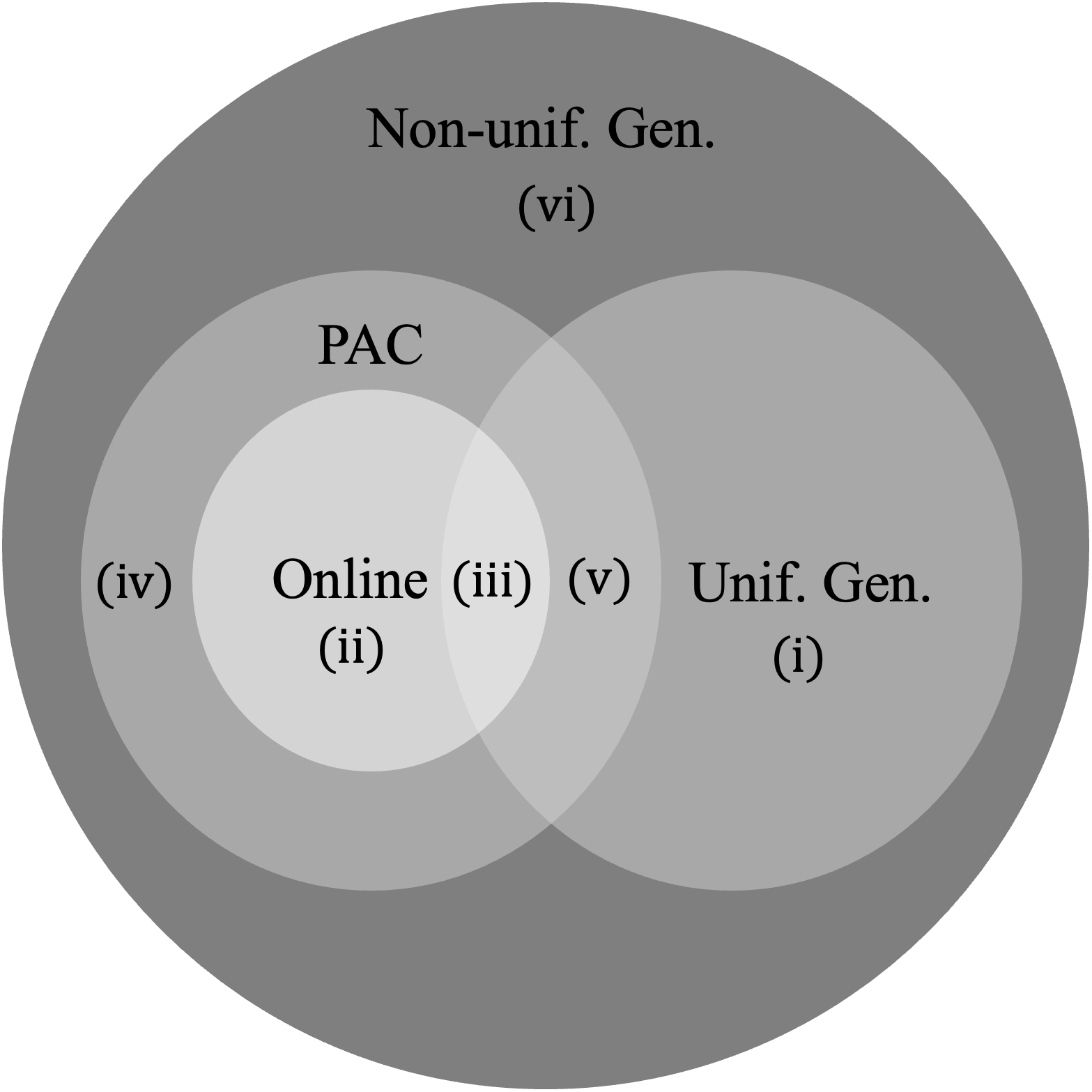}
    \caption{Landscape of Generatability vs. Predictability for countable classes. (i-vi) map to the items in Theorem \ref{thm:genvpred}.}
    \label{fig:pacvsgen}
\end{figure}

Our main result is essentially Figure \ref{fig:pacvsgen} which captures the landscape of generatability and predictability. Its written form is  Theorem \ref{thm:genvpred}, whose proof is in Appendix \ref{app:genvpred}.

\begin{theorem}[Generation v. Prediction] \label{thm:genvpred} Let $\Xcal$ be countable. The following statements are true.

\begin{itemize}
\item[(i)] There exists a countable class $\Hcal \subseteq \{0, 1\}^{\Xcal}$ which is \emph{uniformly generatable} but not \emph{PAC learnable}.
\item[(ii)] There exists a countable class $\Hcal\subseteq \{0, 1\}^{\Xcal}$ which is \emph{online learnable} but not \emph{uniformly generatable}.
\item[(iii)] There exists a countable class $\Hcal\subseteq \{0, 1\}^{\Xcal}$ which is \emph{online learnable} and \emph{uniformly generatable}.
\item[(iv)] There exists a countable class $\Hcal\subseteq \{0, 1\}^{\Xcal}$ that is \emph{PAC learnable} but neither \emph{online learnable} nor \emph{uniformly generatable}.
\item[(v)] There exists a countable class $\Hcal\subseteq \{0, 1\}^{\Xcal}$ that is \emph{PAC learnable} and \emph{uniformly generatable}, but not \emph{online learnable}.
\item[(vi)] There exists a countable class $\Hcal\subseteq \{0, 1\}^{\Xcal}$ that is neither \emph{PAC learnable} nor \emph{uniformly generatable}.
\end{itemize}
\end{theorem}

 Theorem \ref{thm:genvpred} (and Figure \ref{fig:pacvsgen}) shows that even amongst countable classes $\Hcal$ (which we know, via Corollary \ref{cor:countnonunif}, are always non-uniformly generatable), uniform generatability and prediction are truly incomparable -- knowing whether a class is PAC or online learnable tells you \emph{nothing} about whether it is uniformly generatable, and vice versa. As such, these are two fundamentally different properties of a hypothesis class.  Perhaps the best evidence of this is the difference in their behavior under unions. PAC and online learnability behave very nicely under unions -- if $\Hcal_1$ and $\Hcal_2$ are PAC/online learnable, then their union $\Hcal_1 \cup \Hcal_2$ is also PAC/online learnable \citep{dudley1978central, alon2020closure}. However, the same cannot be said for generation -- both uniform generation and non-uniform generation are \emph{not} closed under finite unions. 






\begin{lemma} \label{lem:nonunifclos} Let $\Xcal$ be countable. There exists two classes $\Hcal_1, \Hcal_2 \subseteq \{0, 1\}^{\Xcal}$ such that $\operatorname{C}(\Hcal_1) = \operatorname{C}(\Hcal_2) = 0$, but $\Hcal_1 \cup \Hcal_2$ is not non-uniformly generatable. 
\end{lemma}

\begin{proof} Let $\Xcal = \mathbbm{Z}$ and consider $\Hcal_1 = \{x \mapsto \mathbbm{1}\{x\in A\text{ or } x\leq 0\}:A\subseteq\naturals\}$ and $\Hcal_2 = \{x \mapsto \mathbbm{1}\{x \in \naturals\}\}.$ Its not hard to see that $\operatorname{C}(\Hcal_1) = \operatorname{C}(\Hcal_2) = 0$. Moreover, by Lemma \ref{lem:notnonunifgen}, we know that $\Hcal = \Hcal_1 \cup \Hcal_2$ is not non-uniformly generatable. 
\end{proof}

Lemma \ref{lem:nonunifclos} actually shows something stronger -- there exists a uniformly generatable class $\Hcal$ and a hypothesis $h : \Xcal \rightarrow \{0, 1\}$ such that $\Hcal \cup \{h\}$ is no longer non-uniformly generatable! Such divergent behavior is not present in PAC and online learnability, where the addition of a single hypothesis never changes the learnability of the class. Surprisingly, a similar, but weaker statement can be said about generatability in the limit. 

 \begin{lemma} \label{lem:hardgeninlim} Let $\Xcal$ be countable. There exists a countable sequence of classes $\Hcal_1, \Hcal_2, \dots$, all satisfying the \emph{UUS} property, such that $\operatorname{C}(\Hcal_i) = 0$ for all $i \in \mathbbm{N}$, but the class $\Hcal = \bigcup_{i \in \mathbbm{N}} \Hcal_i$ is \emph{not} generatable in the limit.
\end{lemma}

\begin{proof}
    Let $\Xcal=\mathbb{Q}_+$ be all the positive rational numbers. Let $P=\{p_n\}_{n=1}^\infty$ be the set of prime numbers, indexed in increasing order. For each $i\in\naturals$, define
    \[
    Q_i=\Bigl\{\frac{p}{p_i},p\in P\Bigl\},\quad 
    \Hcal_i=\{x \mapsto\mathbbm{1}\{x\in Q_i\text{ or }x\in A\}:A\in 2^{\mathbb{Q}_+}\}.
    \]
    Note that $\operatorname{C}(\Hcal_i)=0$ for all $i \in \naturals$. We now show that $\Hcal=\bigcup_{i=1}^\infty\Hcal_i$ is not generatable in the limit.

    Suppose $\Hcal$ is generatable in the limit and $\Gcal$ is such a generator, we now prove a contradiction. Let $h_2\in \Hcal_2$ be such that $\operatorname{supp}(h_2)=Q_2\cup\{\frac{3}{2}\}$. Let $\{x_{2,t}\}_{t=1}^\infty$ be an enumeration of $\operatorname{supp}(h_2)$ such that
    \[\{x_{2,1},x_{2,2},x_{2,3},x_{24},\dots\}=\Bigl\{\frac{p_2}{p_1},\frac{p_1}{p_2},\frac{p_2}{p_2},\frac{p_3}{p_2},\dots\Bigl\}.\]
    
    By definition, there exists a time $t_2\ge 4$ such that $\Gcal(x_{21},\dots, x_{2t_2})\in \operatorname{supp}(h_2)\setminus\{x_{21},\dots, x_{2,t_2}\}$, which implies $\Gcal(x_{21},\dots,x_{2,t_2})\in Q_2\setminus\{x_{2,1},\dots, x_{2,t_2}\}$. Note that $t_2\ge 4$ implies that $1\in \{x_{2,1},\dots, x_{2,t_2}\}$. 
    Let $h_3\in \Hcal_3$ be such that $\operatorname{supp}(h_3)=Q_3\cup\{\frac{5}{2}\}\cup\{x_{2,1},\dots, x_{2,t_2}\}$.
     Let $\{x_{3,t}\}_{t=1}^\infty$ be an enumeration of $\operatorname{supp}(h_3)$ such that
     \[\{x_{3,1},\dots,x_{3,t_2}\}=\{x_{2,1},\dots,x_{2,t_2}\}\]
     and 
     \[\{x_{3,t_2+1},x_{3,t_2+2},x_{3,t_2+3},x_{3,t_2+4},\dots\}=\Bigl\{\frac{p_3}{p_1},\frac{p_1}{p_3},\frac{p_2}{p_3},\frac{p_3}{p_3},\dots\Bigl\}.\]
     
     We know there exists a time $t_3\ge t_2+1$ such that $\Gcal(x_{3,1},\dots, x_{3,t_3})\in \operatorname{supp}(h_3)\setminus\{x_{3,1},\dots x_{3,t_3}\}$, which means $\Gcal(x_{3,1},\dots,x_{3,t_3})\in Q_3\setminus\{x_{3,1},\dots x_{3,t_3}\}$.

     Inductively, suppose $h_2,\dots,h_n$ and $t_2,\dots,t_n$ are all defined. Let $h_{n+1}\in \Hcal_{n+1}$ be such that
     \[
     \operatorname{supp}(h_{n+1})=Q_{n+1}\cup\{\frac{p_{n+1}}{p_1}\}\cup\{x_{n,1},\dots,x_{n,t_n}\}.
     \]
     Let $\{x_{n+1,t}\}_{t=1}^\infty$ be an enumeration of $\operatorname{supp}(h_{n+1})$ such that 
     \[\{x_{n+1,1},\dots,x_{n+1,t_n}\}=\{x_{n,1},\dots,x_{n,t_n}\}\]
     and 
     \[
     \{x_{n+1,t_n+1},x_{n+1,t_n+2},x_{n+1,t_n+3},x_{n+1,t_n+4},\dots\}=\Bigl\{\frac{p_n}{p_1},\frac{p_1}{p_n},\frac{p_2}{p_n},\frac{p_3}{p_n},\dots\Bigl\}.
     \]
    By our construction, there exists a time $t_{n+1}\ge t_n+1$ such that $\Gcal(x_{n+1,1},\dots,x_{n+1,t_{n+1}})\in Q_{n+1}\setminus\{x_{n+1,1},\dots, x_{n+1,t_{n+1}}\}$
    Now, let $h_1$ be a hypothesis such that
    \[\operatorname{supp}(h_1)=\bigcup_{i=2}^\infty\{x_{i,1},\dots, x_{i,t_i}\}.\]
    It is clear that $h_1\in \Hcal_1$. Moreover, let $\{x_i\}_{i=1}^\infty$ be an enumeration of $h_1$ such that for all $n \geq 2$,
     \[\{x_1,x_2,\dots,x_{t_n}\}=\{x_{n,1},x_{n,2},\dots,x_{n,t_n}\}.\]
    Then, by our construction, we know that for all $n \geq 2$, $\Gcal(x_1,x_2,\dots,x_{t_n})\in Q_{n}\setminus \operatorname{supp}(h_1)$, which means $h_1$ is not generatable in the limit. 
\end{proof}

We note that Lemma \ref{lem:hardgeninlim} provides two contributions. First, it gives a non-trivial class which is not generatable in the limit. Second, it establishes that even generatability in the limit, the \emph{weakest} definition of generatability, is not very well behaved under (countable) unions. An interesting question is whether generatability in the limit is even closed under \emph{finite} unions. We leave this as an open question in Section \ref{sec:diss}.

\section{Extension to Prompted Generation} \label{sec:promptedgen}

So far, generators are only required to eventually produce new positive examples. This setup does not account for the fact that in many real-life situations, we would like to generate objects with respect to a \emph{prompt}. For example, we may like to generate an image based on a text caption, respond to a specific query from the user, or generate a molecule with particular structural properties. 

To capture a prompted-version of generation and inspired by recent work on multiclass learning \citep{Brukhimetal2022, hanneke2023multiclass}, we let $\Ycal$ be an abstract prompt space and $\Hcal \subseteq \Ycal^{\Xcal}$ be a \emph{multiclass} hypothesis class. Given a hypothesis $h \in \Hcal$ and a prompt $y \in \Ycal$, one should think of the set $\{x \in \Xcal: h(x) = y\}$ as the collection of valid generatable objects for prompt $y$ with respect to hypothesis $h$.  Roughly speaking, if $h \in \Hcal$ captures the true world, and the prompt on round $t$ is $y_t \in \Ycal$, then the goal of the generator should be to output an example $\hat{x}_t$ such that $h(\hat{x}_t) = y_t$.

To handle prompts, we need a slight  modification of the game defined in Section \ref{sec:prelim}. As in the binary case, the adversary first selects a hypothesis $h \in \Hcal \subseteq \Ycal^{\Xcal}$ and a sequence of examples $x_1, x_2, \dots$. But now, it also selects a sequence of prompts $y_1, y_2, \dots.$ In each round $t \in \naturals$, the adversary reveals the tuple  $(x_t, h(x_t), y_t)$ and the goal of the learner is output $\hat{x}_t \in \{x \in \Xcal: h(x) = y_t\} \setminus \{x_1, \dots, x_t\}.$ 


\begin{remark}
Although we call $\Ycal$ a ``prompt" space, we do not place any assumption on $\Ycal$ or its size. Thus, one can place other natural restrictions on generation by choosing $\Ycal$ appropriately. For example, if one takes $\Ycal$ to be a partition of $\Xcal$ and $\Hcal \subseteq \{0, 1\}^{\Xcal}$ to be a binary class, then the corresponding setting can capture generation with some sort of diversity constraint. 
\end{remark}

To make the outlined notion of prompted generatability more formal, we need an extension of the ``support" and UUS property to the prompted setting. For any $h \in \Hcal \subseteq \Ycal^{\Xcal}$ and any $y \in \Ycal$, define the $y$-support of $h$ as:
$$\operatorname{supp}(h, y) := \{x \in \Xcal: h(x) = y\}.$$
Then, the Prompted Uniformly Unbounded Support property just requires that for every $h \in \Hcal$ and any $y \in \Ycal$, the $y$-support of $h$ is unbounded. 

\begin{assumption} [Prompted Uniformly Unbounded Support (PUUS)] A hypothesis class $\Hcal \subseteq \Ycal^{\Xcal}$ satisfies the \emph{Prompted Uniformly Unbounded Support (PUUS)} property if for every $y \in \Ycal$ and $h \in \Hcal$, we have that $|\operatorname{supp}(h, y)| = \infty$. 
\end{assumption}

\begin{remark}
Like the \emph{UUS} property, the \emph{PUUS} property is  only needed for bookkeeping purposes to prevent the adversary from presenting the generator with an impossible task (i.e. generating new examples for prompt $y \in \Ycal$ when no new, examples exist.) One can remove the \emph{PUUS} restriction, but restrict the adversary to choose a prompted sequence so that the generator is always guaranteed the existence of new, unseen examples with the selected prompt $y$. This assumption is also captured in Section 7 of \cite{kleinberg2024language}, where they assume that the adversaries only reveals ``non-trivial" prompts to the generator.
\end{remark}

Next, we extend the definition of Generator to a Prompted Generator. 

\begin{definition}[Prompted Generator] A prompted generator is a map $\Gcal: (\Xcal \times \Ycal \times \Ycal)^{\star} \rightarrow \Xcal$ that takes a finite sequence of tuples $(x_1, h(x_1), y_1), (x_2, h(x_1),  y_2), \dots$ and outputs an example $x$.
\end{definition}

Then, we can define the following notion of prompted uniform generatability. 

\begin{definition}[Prompted Uniform Generatability] \label{def:mcunifgen} Let $\Hcal \subseteq \Ycal^{\Xcal}$  be any hypothesis class satisfying the $\operatorname{PUUS}$ property. Then, $\Hcal$ is prompted \emph{uniformly} generatable, if there exists a prompted generator $\Gcal$ and a number $d^{\star} \in \naturals$, such that for every $h \in \Hcal$, any  sequence $(x_i, y_i)_{i \in \naturals}$, and any $y^{\star} \in \Ycal$, if there exists $t^{\star} \in \naturals$ such that 
$$|\{x_1, \dots, x_{t^{\star}}\} \cap \operatorname{supp}(h, y^{\star})| = d^{\star},$$
\noindent then 
$$\Gcal((x_1, h(x_1), y_1), \dots, (x_s,  h(x_s), y_s)) \in  \operatorname{supp}(h, y_s) \setminus \{x_1, \dots, x_s\}.$$
\noindent for all $s \geq t^{\star}$ where $y_s = y^{\star}.$
\end{definition}

The ``non-uniform" and ``in the limit" definitions follow analogously. 

\begin{definition}[Prompted Non-uniform Generatability] \label{def:mcnonunifgen} Let $\Hcal \subseteq \Ycal^{\Xcal}$  be any hypothesis class satisfying the $\operatorname{PUUS}$ property. Then, $\Hcal$ is prompted \emph{non-uniformly} generatable, if there exists a prompted generator $\Gcal$, such that for every $h \in \Hcal$, there exists a number $d^{\star} \in \naturals$, such that for every sequence $(x_i, y_i)_{i \in \naturals}$, and any $y^{\star} \in \Ycal$, if there exists $t^{\star} \in \naturals$ such that 
$$|\{x_1, \dots, x_{t^{\star}}\} \cap \operatorname{supp}(h, y^{\star})| = d^{\star},$$
\noindent then 
$$\Gcal((x_1, h(x_1), y_1), \dots, (x_s,  h(x_s), y_s)) \in  \operatorname{supp}(h, y_s) \setminus \{x_1, \dots, x_s\}.$$
\noindent for all $s \geq t^{\star}$ where $y_s = y^{\star}.$
\end{definition}

\begin{definition}[Prompted Generatability in the Limit] \label{def:mcgeninlim} Let $\Hcal \subseteq \Ycal^{\Xcal}$  be any hypothesis class satisfying the $\operatorname{PUUS}$ property. Then, $\Hcal$ is prompted generatable in the limit, if there exists a prompted generator $\Gcal$, such that for every $h \in \Hcal$, any sequence $(x_i, y_i)_{i \in \naturals}$, and any $y^{\star} \in \Ycal$, if $\operatorname{supp}(h, y^{\star}) \subseteq \{x_1, x_2, \dots\}$, then there exists $t^{\star} \in \naturals$ such that

$$\Gcal((x_1, h(x_1), y_1), \dots, (x_s,  h(x_s), y_s)) \in  \operatorname{supp}(h, y_s) \setminus \{x_1, \dots, x_s\}.$$
\noindent for all $s \geq t^{\star}$ where $y_s = y^{\star}.$
\end{definition}

Roughly speaking, Definitions \ref{def:mcunifgen},  \ref{def:mcnonunifgen}, and \ref{def:mcgeninlim} state that a class $\Hcal \subseteq \Ycal^{\Xcal}$ is prompted generatable if for any prompt $y \in \Ycal$, after a sufficient number of distinct examples with prompt $y$ are observed, one can generate new examples with prompt $y$. Like in the binary case, generators $\Gcal$ that witness Definitions \ref{def:mcunifgen} and \ref{def:mcnonunifgen} have the following nice property -- for any prompt $y \in \Ycal$, once a sufficient number of distinct examples are observed with prompt $y$, $\Gcal$ can produce new, unseen examples for prompt $y$ auto-regressively and without any form of supervision.

\begin{remark} 
When $\Ycal = \{0, 1\}$, our definition of prompted generatability does \emph{not} reduce to definitions of unprompted generatability in Section \ref{sec:prelim}. The reason is two-fold: (1) in prompted generatability, the Generator can now observe both positive and negative \emph{labeled} examples, and (2) the Generator needs to be able to produce both positive and negative examples. In fact, its not hard to see that when $\Ycal = \{0, 1\}$, prompted uniform/non-uniform generatability is stronger than uniform/non-uniform generatability. Nevertheless, the definitions of uniform and non-uniform generatability from Section \ref{sec:prelim} are a special case of the prompted versions if one restricts the adversary to always pick examples $x_1, x_2, \dots$ such that $h(x_t) = 1$ and prompts $y_1 = 1, y_2 = 1, \dots.$   Observing both positive and negative labeled examples has been previously studied in the context of identifiability \citep{gold1967language, dupont1994regular}.
\end{remark}

\subsection{Comparison to \cite{kleinberg2024language}'s Prompted Generation} \label{sec:pgkm}

Our setting of prompted generation generalizes the model of prompting studied in Section 7 of \cite{kleinberg2024language}. Namely, \cite{kleinberg2024language} consider the following model. Let $\Xcal$ be a suffix space, $\Ycal$ be the prompt space, and $\Zcal \subseteq \{y \circ x: x \in \Xcal, y \in \Ycal\}$ be the space of completed prompts, where $\circ$ denotes the concatenation operator. Let $\Lcal = \{L_1, L_2, \dots \}$ denote a language family defined over $\Zcal$. Before the game begins, the adversary picks language $K \in \Lcal$, a sequence of its completed prompts $z_1, z_2, \dots \in K$, and a sequence of prompts $y_1, y_2, \dots.$ On round $t \in \mathbbm{N}$, the adversary reveals $(z_t, y_t)$, and the goal of the generator is to output $\hat{x}_t \in \Xcal$ such that $y_t \circ \hat{x}_t \in K \setminus \{z_1, \dots, z_t\}$. This model is equivalent to our setting after picking $\Hcal = \{h_L: L \in \Lcal\} \subseteq \Ycal^{\Xcal}$ as the hypothesis class, where $h_L: \Xcal \rightarrow \Ycal$ such that $h_L(x) = y$ if and only if $y \circ x \in L.$ 

That said, our definitions of prompted generatability differ from the notion of ``prompted generation in the limit" in Section 7 of \cite{kleinberg2024language}. We highlight the key differences below. Notice that in all our definitions of prompted generatability, the time point after which perfect generation must occur can be prompt specific.  This means that the generator only needs to perfectly generate with respect to a prompt after it has seen a sufficient number of examples with this prompt. This, however, is not the case for the definition of prompted generation in the limit studied by \cite{kleinberg2024language}. In their model, the generator must eventually perfectly complete prompts it may have never seen in the past. In this sense, our notions of prompted generatability are \emph{weaker}.

\subsection{Characterizations of Prompted Generatability}

To characterize prompted uniform and non-uniform generatability, we extend the Closure dimension to the prompted case. To do so, we need an extension of the closure operator to a \emph{prompted} closure operator $\langle \cdot, \cdot \rangle_{\Hcal}.$ Namely, for every finite sequence of examples $x_1, \dots, x_n$ and prompt $y \in \Ycal$, we define 
$$\Hcal(x_{1:n},y):=\{h\in\Hcal:h(x_i)=y,i=1,\dots,n\}$$
and
$$\langle x_{1:n}, y \rangle_{\Hcal} := \begin{cases}
			\bigcap_{h \in \Hcal(x_{1:n}, y)} \operatorname{supp}(h, y), & \text{if $|\Hcal(x_{1:n}, y)| \geq 1$}\\
            \bot, & \text{if $|\Hcal(x_{1:n}, y)| = 0$}
		 \end{cases}$$

We are now ready to define the Prompted Closure dimension.

\begin{definition}[Prompted Closure dimension] \label{def:gem} The \emph{Prompted Closure dimension} of $\Hcal$, denoted $\operatorname{PC}(\Hcal)$, is the largest number $d \in \naturals$ for which there exists \emph{distinct} $x_1, \dots, x_d \in \Xcal$ and a prompt $y \in \Ycal$ such that $|\langle (x_1, \dots, x_d), y\rangle_{\Hcal}| \neq \bot$ and $|\langle (x_1, \dots, x_d), y\rangle_{\Hcal}| < \infty.$  If this is true for arbitrarily large $d \in \naturals$, then we say that $\operatorname{PC}(\Hcal) = \infty.$ On the other hand, if this is not true for $d = 1$, we say that $\operatorname{PC}(\Hcal) = 0.$
\end{definition}

Using analogous techniques, one can prove that finiteness of the Prompted Closure dimension is both necessary and sufficient for prompted uniform generatability. 

\begin{theorem}[Characterization of Prompted Uniform Generatability] \label{thm:mcunifgen} Let $\Xcal$ and $\Ycal$ be countable. Let $\Hcal \subseteq \Ycal^{\Xcal}$ be any hypothesis class satisfying the $\operatorname{PUUS}$ property. The following statements are equivalent. 

\begin{itemize}
    \item[(i)] $\Hcal$ is prompted uniformly generatable.
    \item[(ii)] $\operatorname{PC}(\Hcal) < \infty.$
\end{itemize}
\end{theorem}

Likewise, the same characterization of non-uniform generatability for the binary case also goes through when considering the prompted setting. 

\begin{theorem}[Characterization of Prompted Non-uniform Generatability] \label{thm:mcnonunifgen} Let $\Xcal$ and $\Ycal$ be countable. Let $\Hcal \subseteq \Ycal^{\Xcal}$ be any hypothesis class satisfying the $\operatorname{PUUS}$ property. The following statements are equivalent. 

\begin{itemize}
    \item[(i)] $\Hcal$ is prompted non-uniformly generatable.
    \item[(ii)] There exists a non-decreasing sequence of classes $\Hcal_1 \subseteq \Hcal_2 \subseteq \dots$  such that $\Hcal =\bigcup_{n\in\naturals} \Hcal_n$ and $\operatorname{PC}(\Hcal_n)<\infty$ for every $n\in\naturals.$
\end{itemize}
\end{theorem}


We highlight that we make no assumptions about the size $\Ycal$ in Theorems \ref{thm:mcunifgen} and \ref{thm:mcnonunifgen} apart from its countableness. That is, these theorems hold even when $\Ycal$ is countably infinite. The proofs of Theorem \ref{thm:mcunifgen} and \ref{thm:mcnonunifgen} are very similar to that of Theorem \ref{thm:unifgen} and \ref{thm:nonunifgen}. For the sake of conciseness, we omit the details in the main text, but provide proof sketches in the Appendix. When $|\Ycal| < \infty$,  we can show that all finite classes are prompted uniformly generatable and all countable classes are prompted non-uniformly generatable.  The latter also implies that all countable classes are prompted generatable in the limit. 

\begin{corollary} \label{thm:mcfiniteclasses} Let $\Xcal$ be countable and $\Ycal$ be finite. Let $\Hcal \subseteq \Ycal^{\Xcal}$  be any hypothesis class satisfying the $\operatorname{PUUS}$ property. The following statements are true. 

\begin{itemize}
\item[(i)] $|\Hcal| < \infty$  $\implies$ $\Hcal$ is prompted uniformly generatable. 
\item[(ii)]  $\Hcal$ is countably infinite $\implies$ $\Hcal$ is prompted non-uniformly generatable. 
\item[(iii)] $\Hcal$ is countably infinite $\implies$ $\Hcal$ is prompted  generatable in the limit. 
\end{itemize}
\end{corollary}

\begin{proof}
(ii) follows from (i) and Theorem \ref{thm:mcnonunifgen} and (iii) follows from (ii), so we only focus on proving (i). Let $\Hcal$ be any finite class. Fix some $y \in \Ycal$ and consider the binary hypothesis class:

$$\Hcal_y := \{ x \mapsto \mathbbm{1}\{h(x) = y\}: h \in \Hcal\}.$$

\noindent Since $\Hcal$ is finite, so is $\Hcal_y.$ Accordingly, by Theorem \ref{thm:genfinite}, we know that there exists $d_y \in \mathbbm{N}$ that witnesses the fact that $\Hcal_y$ is uniform generatable according to Definition \ref{def:unifgen}. We claim that $\operatorname{PC}(\Hcal) = \max_{y \in \Ycal} d_y.$ For the sake of contradiction, suppose this is not the case. That is, $\operatorname{PC}(\Hcal) \geq (\max_{y \in \Ycal} d_y ) + 1.$ Then, by definition, there exists a distinct sequence $x_1, \dots , x_{\operatorname{PC}(\Hcal)}$ and a prompt $y^{\star} \in \Ycal$ such that $|\langle (x_1, \dots, x_{\operatorname{PC}(\Hcal)}), y\rangle_{\Hcal}| \neq \bot$ and $|\langle (x_1, \dots, x_{\operatorname{PC}(\Hcal)}), y^{\star}\rangle_{\Hcal}| < \infty.$  This implies that $d_{y^{\star}} \geq \operatorname{PC}(\Hcal)$ which contradicts the fact that $\operatorname{PC}(\Hcal) = \max_{y \in \Ycal} d_y.$ This completes the proof, as $\operatorname{PC}(\Hcal) = \max_{y \in \Ycal} d_y < \infty$ and thus, by Theorem \ref{thm:mcunifgen}, $\Hcal$ is prompted uniformly generatable. 
\end{proof}

However, quite surprisingly, we find that this is not the case when $|\Ycal| = \infty$ -- there exists a finite class which is not prompted uniformly generatable!

\begin{lemma} \label{lem:infiniteYnotunifgen} Let $\Xcal$ be countable and $\Ycal$ be countably infinite. There exists a finite class $\Hcal \subseteq \Ycal^{\Xcal}$  satisfying the $\operatorname{PUUS}$ property such that $\Hcal$ is not prompted uniformly generatable. 
\end{lemma}

\begin{proof} Let $\Xcal = \mathbbm{Z}$ and $\Ycal = \mathbbm{N}$. For every $n \in \mathbbm{N}$, define the set

$$A_n := \Biggl\{\frac{n(n-1)}{2}+1,\dots,\frac{n(n-1)}{2}+n\Biggl\}.$$ Let $\{p_n\}_{n \in \mathbbm{N}}$ be the set of all prime numbers. Consider the hypotheses $h_1: \Xcal \rightarrow \Ycal$ and $h_2:\Xcal \rightarrow \Ycal$ defined as

$$h_1(x) := \begin{cases}
			n, & \text{if $x \in A_n$}\\
            n, & \text{if $x \in \{-p_n, -p^2_n, -p^3_n, \dots\}$}\\
            1, & \text{otherwise}
		 \end{cases},$$

\noindent and 

$$h_2(x) := \begin{cases}
			n, & \text{if $x \in A_n$}\\
            n, & \text{if $x \in \{-p_{n+1}, -p^2_{n+1}, -p^3_{n+1}, \dots\}$}\\
            1, & \text{otherwise}
		 \end{cases}.$$

Let $\Hcal = \{h_1, h_2\}.$ Observe that $\Hcal$ satisfies the PUUS property. We will now show that $\Hcal$ is not prompted uniformly generatable. 

By Theorem \ref{thm:mcunifgen}, it suffices to show that $\operatorname{PC}(\Hcal) = \infty$. In particular, it suffices to show that for every $d \in \naturals \setminus \{1\}$, there exists distinct $x_1, \dots ,x_d$ and a prompt $y \in \Ycal$ such that $|\langle (x_1, \dots x_d), y\rangle| \neq \bot$ and $|\langle (x_1, \dots x_d), y\rangle| < \infty.$ To that end, fix some $d \in \naturals \setminus \{1\}.$ Consider the sequence $x_1, \dots, x_d$ obtained by sorting $A_d$ in increasing order and consider the prompt $y = d.$ Then, we have that $\Hcal(x_{1:d}, y) = \Hcal$ and 

$$\langle x_{1:d}, y \rangle_{\Hcal}  = \bigcap_{h \in \Hcal} \operatorname{supp}(h, d) = A_d,$$

\noindent so that $|\langle x_{1:d}, y \rangle_{\Hcal}| < \infty.$ The proof is complete after noting that $d \in \naturals \setminus \{1\}$ is picked arbitrarily. 
\end{proof}

Moreover, we can use Lemma \ref{lem:infiniteYnotunifgen} and Theorem \ref{thm:mcnonunifgen}  to show the existence of a finite class which is not prompted non-uniformly generatable.

\begin{corollary} \label{lem:mcfiniteclasses} Let $\Xcal$ be countable and $\Ycal$ be countably infinite. There exists a finite class $\Hcal \subseteq \Ycal^{\Xcal}$  satisfying the $\operatorname{PUUS}$ property such that $\Hcal$ is not prompted non-uniformly generatable. 
\end{corollary}

\begin{proof} Consider the same class $\Hcal = \{h_1, h_2\}$ from the proof of Lemma \ref{lem:infiniteYnotunifgen}. We know that $\operatorname{PC}(\Hcal) = \infty.$ By Theorem \ref{thm:mcnonunifgen}, $\Hcal$ is not prompted non-uniformly generatable if for every non-decreasing sequence of classes $\Hcal_1 \subseteq \Hcal_2 \subseteq \dots$ satisfying $\Hcal =\bigcup_{n\in\naturals} \Hcal_n$, there exists a $i \in \naturals$ such that $\operatorname{PC}(\Hcal_i)= \infty$. Trivially, for every non-decreasing sequence of classes $\Hcal_1 \subseteq \Hcal_2 \subseteq \dots$ satisfying $\Hcal =\bigcup_{n\in\naturals} \Hcal_n$, there must be an index $i \in \naturals$ such that $\Hcal = \Hcal_i.$ Since $\operatorname{PC}(\Hcal) = \infty$, our proof is complete. 
\end{proof}

We leave as an open question whether such a separation exists for prompted generatability in the limit. These results highlight that the behavior of prompted generatability changes significantly when the prompt space is allowed to be unbounded. This sort of phase change is not unique to generation, and has also been observed in the context of multiclass PAC learning. For example, it is well known that uniform convergence is not necessary for multiclass PAC learnability when $|\Ycal| = \infty$ \citep{natarajan1992probably}. This is contrast to when $|\Ycal| < \infty$, where uniform convergence provides a characterization of multiclass PAC learnability \citep{DanielyERMprinciple, daniely2014optimal}. Similar observations have been made in  multiclass online classification \citep{hanneke2023multiclass} and multiclass transductive online classification \citep{hanneke2024multiclass}.

\section{Discussion and Future Directions} \label{sec:diss}
In this work, we reinterpreted the results by \cite{gold1967language} \cite{angluin1979finding}, \cite{angluin1980inductive}, and \cite{kleinberg2024language} through the lens of learning theory. By doing so, we are able to formalize three paradigms in learning theory: generation in the limit, non-uniform generation, and uniform generation. By abstracting the problem of generation to an arbitrary example space and binary hypothesis classes, we are able to study the fundamental nature of generation beyond language modeling and compare generation with prediction. We end by highlighting  several important directions for future work.

\vspace{5pt}

\noindent \textbf{Characterizing Generatability in the Limit.} \cite{kleinberg2024language} proved that every countable hypothesis class is generatable in the limit. We gave an alternate sufficiency condition which showed the existence of many uncountably infinite classes that are generatable in the limit. However, it is unclear (and unlikely) that our sufficiency condition, in conjunction with countableness, provides a characterization of generatability in the limit.  This motivates our first open question. 

\begin{question} What characterizes generatability in the limit?
\end{question}

\noindent Ideally, a characterization of generatability in the limit can be written neatly in set-theoretic language like Angluin's characterization of identifiability in the limit.  

\vspace{5pt}

\noindent \textbf{Generatability in the Limit under Finite Unions.} In Section \ref{sec:genvpred}, we  showed that uniform and non-uniform generatability are not closed under finite unions. However, we were able to only show that generatability in the limit is not closed under countable unions. This motivates our second open question. 

\begin{question} \label{ques:geninlim} Is generatability in the limit closed under finite unions? 
\end{question}

Resolving Question \ref{ques:geninlim} is important for two reasons: (1) it may lead to a complete characterization of generatability in the limit and (2) it may provide insight on how to optimally combine generators. Recall that Theorem \ref{thm:weaksuff} shows that the finite union of uniformly generatable classes are generatable in the limit. Thus, as a first step towards resolving Question \ref{ques:geninlim}, it might be helpful to resolve the following open question.

\begin{question} \label{ques:nonunif2geninlim} Is the finite union of non-uniformly generatable classes generatable in the limit?
\end{question}

\noindent \textbf{Characterizing Prompted Generatability in the Limit.} In Section \ref{sec:promptedgen}, we provided complete characterizations of prompted uniform and non-uniform generation. When $|\Ycal| < \infty$, we showed that all countable classes are prompted generatable in the limit. However, we left open the complete characterization of prompted generatability in the limit, which motivates our first question.  

\begin{question} \label{ques:finiteYprompGIL} When $|\Ycal| < \infty$, what characterizes prompted generatability in the limit?
\end{question}

When $|\Ycal| = \infty$, we show that there are finite classes which are not even prompted non-uniformly generatable. This begs the question of whether all countable classes continue to be prompted generatable in the limit when $|\Ycal| = \infty.$

\begin{question} \label{ques:finiteYprompGIL} When $|\Ycal| = \infty$, are all countable classes prompted generatable in the limit?
\end{question}

Unlike the case for prompted uniform and non-uniform generatability, we conjecture that all countable classes are still prompted generatable in the limit when $|\Ycal| = \infty$. Our  claim is due to the positive result by \cite{kleinberg2024language}, who show that in their model of prompting, which can be stronger than ours (see Section \ref{sec:pgkm}), all countable classes are still prompted generatable in the limit.

\bibliographystyle{abbrvnat}
\bibliography{references}

\begin{thebibliography}{34}
\providecommand{\natexlab}[1]{#1}
\providecommand{\url}[1]{\texttt{#1}}
\expandafter\ifx\csname urlstyle\endcsname\relax
  \providecommand{\doi}[1]{doi: #1}\else
  \providecommand{\doi}{doi: \begingroup \urlstyle{rm}\Url}\fi

\bibitem[Alon et~al.(2020)Alon, Beimel, Moran, and Stemmer]{alon2020closure}
N.~Alon, A.~Beimel, S.~Moran, and U.~Stemmer.
\newblock Closure properties for private classification and online prediction.
\newblock In \emph{Conference on Learning Theory}, pages 119--152. PMLR, 2020.

\bibitem[Angluin(1979)]{angluin1979finding}
D.~Angluin.
\newblock Finding patterns common to a set of strings.
\newblock In \emph{Proceedings of the eleventh annual ACM Symposium on Theory of Computing}, pages 130--141, 1979.

\bibitem[Angluin(1980)]{angluin1980inductive}
D.~Angluin.
\newblock Inductive inference of formal languages from positive data.
\newblock \emph{Information and control}, 45\penalty0 (2):\penalty0 117--135, 1980.

\bibitem[Ben-David et~al.(2009)Ben-David, P{\'a}l, and Shalev-Shwartz]{ben2009agnostic}
S.~Ben-David, D.~P{\'a}l, and S.~Shalev-Shwartz.
\newblock Agnostic online learning.
\newblock In \emph{COLT}, volume~3, page~1, 2009.

\bibitem[Bousquet et~al.(2021)Bousquet, Hanneke, Moran, Van~Handel, and Yehudayoff]{bousquet2021theory}
O.~Bousquet, S.~Hanneke, S.~Moran, R.~Van~Handel, and A.~Yehudayoff.
\newblock A theory of universal learning.
\newblock In \emph{Proceedings of the 53rd Annual ACM SIGACT Symposium on Theory of Computing}, pages 532--541, 2021.

\bibitem[Brukhim et~al.(2022)Brukhim, Carmon, Dinur, Moran, and Yehudayoff]{Brukhimetal2022}
N.~Brukhim, D.~Carmon, I.~Dinur, S.~Moran, and A.~Yehudayoff.
\newblock A characterization of multiclass learnability, 2022.
\newblock URL \url{https://arxiv.org/abs/2203.01550}.

\bibitem[Cesa-Bianchi and Lugosi(2006)]{cesa2006prediction}
N.~Cesa-Bianchi and G.~Lugosi.
\newblock \emph{Prediction, learning, and games}.
\newblock Cambridge university press, 2006.

\bibitem[Charikar and Pabbaraju(2024)]{charikar2024exploring}
M.~Charikar and C.~Pabbaraju.
\newblock Exploring facets of language generation in the limit.
\newblock \emph{arXiv preprint arXiv:2411.15364}, 2024.

\bibitem[Daniely and Shalev-Shwartz(2014)]{daniely2014optimal}
A.~Daniely and S.~Shalev-Shwartz.
\newblock Optimal learners for multiclass problems.
\newblock In \emph{Conference on Learning Theory}, pages 287--316. PMLR, 2014.

\bibitem[Daniely et~al.(2011)Daniely, Sabato, Ben-David, and Shalev-Shwartz]{DanielyERMprinciple}
A.~Daniely, S.~Sabato, S.~Ben-David, and S.~Shalev-Shwartz.
\newblock Multiclass learnability and the erm principle.
\newblock In S.~M. Kakade and U.~von Luxburg, editors, \emph{Proceedings of the 24th Annual Conference on Learning Theory}, volume~19 of \emph{Proceedings of Machine Learning Research}, pages 207--232, Budapest, Hungary, 09--11 Jun 2011. PMLR.

\bibitem[Dudley(1978)]{dudley1978central}
R.~M. Dudley.
\newblock Central limit theorems for empirical measures.
\newblock \emph{The Annals of Probability}, pages 899--929, 1978.

\bibitem[Dupont(1994)]{dupont1994regular}
P.~Dupont.
\newblock Regular grammatical inference from positive and negative samples by genetic search: the gig method.
\newblock In \emph{International Colloquium on Grammatical Inference}, pages 236--245. Springer, 1994.

\bibitem[Gold(1967)]{gold1967language}
E.~M. Gold.
\newblock Language identification in the limit.
\newblock \emph{Information and control}, 10\penalty0 (5):\penalty0 447--474, 1967.

\bibitem[Goodfellow et~al.(2014)Goodfellow, Pouget-Abadie, Mirza, Xu, Warde-Farley, Ozair, Courville, and Bengio]{goodfellow2014generative}
I.~Goodfellow, J.~Pouget-Abadie, M.~Mirza, B.~Xu, D.~Warde-Farley, S.~Ozair, A.~Courville, and Y.~Bengio.
\newblock Generative adversarial nets.
\newblock \emph{Advances in neural information processing systems}, 27, 2014.

\bibitem[Hanneke et~al.(2023)Hanneke, Moran, Raman, Subedi, and Tewari]{hanneke2023multiclass}
S.~Hanneke, S.~Moran, V.~Raman, U.~Subedi, and A.~Tewari.
\newblock Multiclass online learning and uniform convergence.
\newblock \emph{Proceedings of the 36th Annual Conference on Learning Theory (COLT)}, 2023.

\bibitem[Hanneke et~al.(2024)Hanneke, Raman, Shaeiri, and Subedi]{hanneke2024multiclass}
S.~Hanneke, V.~Raman, A.~Shaeiri, and U.~Subedi.
\newblock Multiclass transductive online learning.
\newblock \emph{arXiv preprint arXiv:2411.01634}, 2024.

\bibitem[Harshvardhan et~al.(2020)Harshvardhan, Gourisaria, Pandey, and Rautaray]{harshvardhan2020comprehensive}
G.~Harshvardhan, M.~K. Gourisaria, M.~Pandey, and S.~S. Rautaray.
\newblock A comprehensive survey and analysis of generative models in machine learning.
\newblock \emph{Computer Science Review}, 38:\penalty0 100285, 2020.

\bibitem[Jebara(2012)]{jebara2012machine}
T.~Jebara.
\newblock \emph{Machine learning: discriminative and generative}, volume 755.
\newblock Springer Science \& Business Media, 2012.

\bibitem[Jumper et~al.(2021)Jumper, Evans, Pritzel, Green, Figurnov, Ronneberger, Tunyasuvunakool, Bates, {\v{Z}}{\'\i}dek, Potapenko, et~al.]{jumper2021highly}
J.~Jumper, R.~Evans, A.~Pritzel, T.~Green, M.~Figurnov, O.~Ronneberger, K.~Tunyasuvunakool, R.~Bates, A.~{\v{Z}}{\'\i}dek, A.~Potapenko, et~al.
\newblock Highly accurate protein structure prediction with alphafold.
\newblock \emph{nature}, 596\penalty0 (7873):\penalty0 583--589, 2021.

\bibitem[Kalavasis et~al.(2024)Kalavasis, Mehrotra, and Velegkas]{kalavasis2024limits}
A.~Kalavasis, A.~Mehrotra, and G.~Velegkas.
\newblock On the limits of language generation: Trade-offs between hallucination and mode collapse.
\newblock \emph{arXiv preprint arXiv:2411.09642}, 2024.

\bibitem[Khan et~al.(2022)Khan, Naseer, Hayat, Zamir, Khan, and Shah]{khan2022transformers}
S.~Khan, M.~Naseer, M.~Hayat, S.~W. Zamir, F.~S. Khan, and M.~Shah.
\newblock Transformers in vision: A survey.
\newblock \emph{ACM computing surveys (CSUR)}, 54\penalty0 (10s):\penalty0 1--41, 2022.

\bibitem[Kleinberg and Mullainathan(2024)]{kleinberg2024language}
J.~Kleinberg and S.~Mullainathan.
\newblock Language generation in the limit.
\newblock \emph{arXiv preprint arXiv:2404.06757}, 2024.

\bibitem[Littlestone(1987)]{Littlestone1987LearningQW}
N.~Littlestone.
\newblock Learning quickly when irrelevant attributes abound: A new linear-threshold algorithm.
\newblock \emph{Machine Learning}, 2:\penalty0 285--318, 1987.

\bibitem[Lu(2023)]{lu2023non}
Z.~Lu.
\newblock Non-uniform online learning: Towards understanding induction.
\newblock \emph{arXiv preprint arXiv:2312.00170}, 2023.

\bibitem[Malliaris and Moran(2022)]{malliaris2022unstable}
M.~Malliaris and S.~Moran.
\newblock The unstable formula theorem revisited via algorithms.
\newblock \emph{arXiv preprint arXiv:2212.05050}, 2022.

\bibitem[Mohri et~al.(2012)Mohri, Rostamizadeh, and Talwalkar]{10.5555/2371238}
M.~Mohri, A.~Rostamizadeh, and A.~Talwalkar.
\newblock \emph{Foundations of Machine Learning}.
\newblock The MIT Press, 2012.
\newblock ISBN 026201825X.

\bibitem[Murphy(2023)]{pml2Book}
K.~P. Murphy.
\newblock \emph{Probabilistic Machine Learning: Advanced Topics}.
\newblock MIT Press, 2023.
\newblock URL \url{http://probml.github.io/book2}.

\bibitem[Natarajan(1992)]{natarajan1992probably}
B.~Natarajan.
\newblock Probably approximate learning over classes of distributions.
\newblock \emph{SIAM Journal on Computing}, 21\penalty0 (3):\penalty0 438--449, 1992.

\bibitem[Shalev-Shwartz and Ben-David(2014)]{ShwartzDavid}
S.~Shalev-Shwartz and S.~Ben-David.
\newblock \emph{Understanding Machine Learning: From Theory to Algorithms}.
\newblock Cambridge University Press, USA, 2014.

\bibitem[Vanhaelen et~al.(2020)Vanhaelen, Lin, and Zhavoronkov]{vanhaelen2020advent}
Q.~Vanhaelen, Y.-C. Lin, and A.~Zhavoronkov.
\newblock The advent of generative chemistry.
\newblock \emph{ACS Medicinal Chemistry Letters}, 11\penalty0 (8):\penalty0 1496--1505, 2020.

\bibitem[Vapnik and Chervonenkis(1974)]{vapnik1974theory}
V.~Vapnik and A.~Chervonenkis.
\newblock Theory of pattern recognition, 1974.

\bibitem[Vapnik and Chervonenkis(1971)]{vapnik1971uniform}
V.~N. Vapnik and A.~Y. Chervonenkis.
\newblock On uniform convergence of the frequencies of events to their probabilities.
\newblock \emph{Teoriya Veroyatnostei i ee Primeneniya}, 16\penalty0 (2):\penalty0 264--279, 1971.

\bibitem[Wolf et~al.(2020)Wolf, Debut, Sanh, Chaumond, Delangue, Moi, Cistac, Rault, Louf, Funtowicz, et~al.]{wolf2020transformers}
T.~Wolf, L.~Debut, V.~Sanh, J.~Chaumond, C.~Delangue, A.~Moi, P.~Cistac, T.~Rault, R.~Louf, M.~Funtowicz, et~al.
\newblock Transformers: State-of-the-art natural language processing.
\newblock In \emph{Proceedings of the 2020 conference on empirical methods in natural language processing: system demonstrations}, pages 38--45, 2020.

\bibitem[Zhao et~al.(2023)Zhao, Zhou, Li, Tang, Wang, Hou, Min, Zhang, Zhang, Dong, et~al.]{zhao2023survey}
W.~X. Zhao, K.~Zhou, J.~Li, T.~Tang, X.~Wang, Y.~Hou, Y.~Min, B.~Zhang, J.~Zhang, Z.~Dong, et~al.
\newblock A survey of large language models.
\newblock \emph{arXiv preprint arXiv:2303.18223}, 2023.

\end{thebibliography}

\appendix

\section{Proof of Theorem \ref{thm:genvpred}} \label{app:genvpred}

\subsection{Proof of (i)}

Let $\Xcal = \mathbbm{Z}$ and consider the hypothesis class $\Hcal = \{x \mapsto \mathbbm{1}\{x \in A \text{ or } x \leq 0\}: A \subset \mathbbm{N}, |A| < \infty\}.$ First, note that $\Hcal$ satisfies the $\operatorname{UUS}$ since for every $x \in \mathbbm{Z}_{\leq 0}$ and every $h \in \Hcal$, we have that $h(x) = 1$. Second, its not too hard to see that $\operatorname{VC}(\Hcal) = \infty$ since it can shatter arbitrary length sequence of examples of the form $x_1 = 1, x_2 = 2, \dots, x_d = d$. Finally, observe that $\operatorname{C}(\Hcal) = 0$ since $\langle x \rangle_{\Hcal} = \mathbbm{Z}_{\leq 0}$ for all $x \in \Xcal.$

Such a separation also occurs for the more natural hypothesis class of convex polygons defined over the rationals.  This result is also noted by \cite{kleinberg2024language} in Section 3.2, but we summarize it below. 

 Let $\Xcal = \mathbbm{Q}^2$ and $\Hcal \subseteq \{0, 1\}^{\Xcal}$ be the class of all convex polygons over $\Xcal$. That is $\Hcal := \{h \in \{0, 1\}^{\Xcal}: \operatorname{supp}(h) \text{ is a convex polygon.} \}$.  It is well known that $\operatorname{VC}(\Hcal) = \infty$.  Moreover, its not too hard to see that $\Hcal$ satisfies the UUS property. We now show that $\operatorname{C}(\Hcal) < \infty$. In fact, we can show that $\operatorname{C}(\Hcal) = 0$. Indeed, pick any $x \in \Xcal$, and note that the set $\langle x \rangle_{\Hcal}$  is a convex polygon since $\operatorname{supp}(h)$ is a convex polygon for all $h \in \Hcal$. Accordingly, we have that $\lvert \langle x \rangle_{\Hcal}\lvert = \infty$, completing the proof.    

\subsection{Proof of (ii)}

Let $\Xcal = \mathbbm{Z}$ and consider the same class $\Hcal$ from Lemma \ref{lem:nonunifvsunifgen}. Recall, that $\operatorname{C}(\Hcal) = \infty.$ Thus, it suffices to  show that $\operatorname{L}(\Hcal) = 2.$  First, we prove that $\operatorname{L}(\Hcal) \geq 2$ by showing that $\operatorname{VC}(\Hcal) \geq 2$. Indeed, it is not hard to verify that any $x_1, x_2$ such that $x_1 \in \mathbbm{Z}_{<0}$  and $x_2 \in  \mathbbm{Z}_{>0}$ can be shattered by $\Hcal$, completing this direction.

In fact, we can show that this is the only way that two examples can be shattered.  Pick any two examples $x_1, x_2 \in \mathbbm{Z}$. None of $x_1, x_2$ can be $0$, as otherwise all hypothesis output $0$ on this example. Thus, assume that $x_1, x_2$ and are non-zero. Our proof will now be in cases. 

Suppose that $x_1, x_2 > 0$. For every $d \in \mathbbm{N}$, define $A_d := \{\frac{d(d-1)}{2} + 1, , \dots, \frac{d(d-1)}{2} + d\}$. Note that $A_1, A_2, \dots$ are all pairwise disjoint. If there exists a $d \in \mathbbm{N}$ such that $x_1, x_2$ lie in $A_d$, then the labeling $(1, 0)$ is not possible. If $x_1, x_2$ lie in different sets, then the labeling $(1, 1)$ is not possible. Thus, when $x_1, x_2 > 0$, they cannot be shattered by $\Hcal$. 

Suppose, $x_1, x_2< 0.$ If they are both even, then one cannot get the labeling $(1, 0)$. Likewise for odd. If one of them is even, say $x_1$ without loss of generality, then one cannot get the labeling $(1, 0).$ Likewise for odd. Thus, whenever $x_1, x_2 < 0$, they cannot be shattered by $\Hcal$.

Thus, the only way $x_1, x_2$ can be shattered is if one of them is strictly negative, but the other is strictly positive. 

Using this observation, we will now show that $\operatorname{L}(\Hcal) \leq 2.$  Consider any Littlestone tree $\Tcal$ of depth $3$. We will show that $\Tcal$ cannot be shattered by $\Hcal$. There are two cases to consider. 

Suppose the root node is labeled by a strictly negative integer. Then, using the above analysis it must be the case that the nodes on the second level must be labeled by strictly positive integers in order for $\Tcal$ to be shattered. Now, pick any root-to-leaf path prefix $b = (1, 1)$ down $\Tcal$. Let $(x_1, 1), (x_2, 1) $ denote the sequence of labeled examples obtained by traversing down $\Tcal$ according to the prefix $b$. By our observation above, we know that $x_1 < 0$ and $x_2 > 0$. However, there can be exactly one hypothesis $h \in \Hcal$ that is consistent with $(x_1, 1)$ and $(x_2, 1)$. Thus, any completion of the path $b$ cannot be shattered by $\Hcal$. 

Suppose the root node is labeled by a strictly positive integer. Then,  using the above analysis, it must be the case that the nodes on the second level must be labeled by strictly negative integers in order for $\Tcal$ to be shattered. Now, pick the root-to-leaf path prefix $b = (1, 1)$ down $\Tcal$. Let $(x_1, 1), (x_2, 1) $ denote the sequence of labeled examples obtained by traversing down $\Tcal$ according to the prefix $b$. By our observation above, we know that $x_1 > 0$ and $x_2 < 0$. However, there can be exactly one hypothesis $h \in \Hcal$ that can be consistent with $(x_1, 1)$ and $(x_2, 1)$. Thus, any completion of the path $b$ cannot be shattered by $\Hcal$. Since $\Tcal$ was arbitrary, it must be the case that $\operatorname{L}(\Hcal) \leq 2.$

\subsection{Proofs of (iii)-(vi)}

To see (iii), observe that the class $\Hcal = \{x \mapsto \mathbbm{1}\{x = a \text{ or } x \leq 0\}: a \in \mathbbm{N}\}$ is trivially online learnable and uniformly generatable. To see (iv), observe that the class $\Hcal_{\text{thresh}} \cup \Hcal^e \cup \Hcal^o$ is PAC learnable but neither uniformly generatable nor online learnable, where $\Hcal_{\text{thresh}} := \{x \mapsto \mathbbm{1}\{x \geq a\}: a \in \naturals\}$ and $\Hcal^e$,$\Hcal^o$ are defined in Equations \ref{eq:He} and \ref{eq:Ho}. To see (v), observe that the class $\Hcal_{\text{thresh}}$ is PAC learnable, uniformly generatable, but not online learnable. Finally, to see (vi), consider the class $\Hcal = \{x \mapsto \mathbbm{1}\{x \notin A\}: A \subset \naturals, |A| < \infty\}.$

\section{Proofs for Prompted Generatability}
\subsection{Proof sketches of Theorem \ref{thm:mcunifgen}}

\begin{proof} (sketch of sufficiency) Let $\Hcal \subseteq \Ycal^{\Xcal}$ be such that it satisfies the PUUS property and $\operatorname{PC}(\Hcal) < \infty$. Consider the following prompted generator $\Gcal$. For any finite sequence of tuples $(x_1, p_1, y_1), \dots, (x_t, p_t, y_t) \in (\Xcal, \Ycal, \Ycal)^{t}$, $\Gcal$ extracts  $B_t := \{x_i : p_i = y_t\}$,  the subset of examples where $p_i$ is $y_t$. Then, $\Gcal$ checks whether $|B_t| \geq \operatorname{PC}(\Hcal) + 1$. If so, $\Gcal$ computes $\langle B_t, y_t \rangle_{\Hcal}$ and plays $\hat{x}_t \in \langle B_t, y_t \rangle_{\Hcal} \setminus \{x_1, \dots, x_t\}.$ Otherwise, $\Gcal$, plays an arbitrary $\hat{x}_t \in \Xcal.$ We claim that $\Gcal$ is a prompted uniform generator for $\Hcal$. To see this, let $h\in \Hcal$ and $(x_1, h(x_1), y_1), (x_2, h(x_2), y_2), \dots$ be the hypothesis and sequence of tuples chosen by the adversary. Fix an arbitrary reference label $y^{\star} \in \Ycal$. It suffices to show that if there exists a $t^{\star} \in \mathbbm{N}$ such that 

$$|\{x_1, \dots, x_{t^{\star}}\} \cap \operatorname{supp}(h, y^{\star})| = \operatorname{PC}(\Hcal) + 1,$$

\noindent then 

$$\Gcal((x_1, h(x_1), y_1), \dots, (x_s,  h(x_s)), y_s) \in  \operatorname{supp}(h, y_s) \setminus \{x_1, \dots, x_s\}$$

\noindent for all $s \geq t^{\star}$ where $y_s = y^{\star}.$ To that end, suppose there exists a $t^{\star} \in \mathbbm{N}$ such that 

$$|\{x_1, \dots, x_{t^{\star}}\} \cap \operatorname{supp}(h, y^{\star})| = \operatorname{PC}(\Hcal) + 1.$$ Fix an arbitrary $s \geq t^{\star}$ such that $y_s = y^{\star}.$ By construction, we have that $|B_s| \geq \operatorname{PC}(\Hcal) + 1.$ Accordingly, $\Gcal$ computes $\langle B_s, y_s \rangle_{\Hcal} = \langle B_s, y^{\star} \rangle_{\Hcal}$. By definition of the Prompted Closure dimension, it must be the case that $|\langle B_s, y_s \rangle_{\Hcal}| = \infty$. Accordingly, $\langle B_s, y_s \rangle_{\Hcal} \setminus \{x_1, \dots, x_s\} \neq \emptyset$, and we have that $\hat{x}_s \in  \operatorname{supp}(h, y_s) \setminus \{x_1, \dots, x_s\}$, completing the proof. 
\end{proof}

\begin{proof} (sketch of necessity) Let $\Hcal \subseteq \Ycal^{\Xcal}$ be such that it satisfies the MUUS property and $\operatorname{PC}(\Hcal) = \infty$. Let $\Gcal$ be any prompted generator. It suffices to show that for arbitrarily large $d \in \mathbbm{N}$, there exists a sequence of distinct examples $x_1, \dots, x_d$, a prompt $y \in \Ycal$, and a hypothesis $h \in \Hcal$ such that $h(x_i) = y$ for all $i \in [d]$, and 

$$\Gcal((x_1, h(x_1), y), \dots, (x_d,  h(x_d)), y) \notin  \operatorname{supp}(h, y) \setminus \{x_1, \dots, x_d\}.$$

By definition of the Prompted Closure dimension, for every $n \in \mathbbm{N}$ there exists distinct $x_1, \dots, x_n \in \Xcal$ and a label $y^{\star} \in \Ycal$ such that $|\langle (x_1, \dots, x_n), y^{\star}\rangle_{\Hcal}| \neq \bot$ and $|\langle (x_1, \dots, x_n), y^{\star}\rangle_{\Hcal}| < \infty.$ Thus, for some $d \geq n$, there exists a distinct $x_1, \dots, x_d \in \Xcal$ such that $|\langle (x_1, \dots, x_d), y^{\star}\rangle_{\Hcal}| \neq \bot$ and $|\langle (x_1, \dots, x_d), y^{\star}\rangle_{\Hcal} \setminus \{x_1, \dots, x_d\}| = 0.$ Accordingly, for every $x \in \Xcal \setminus \{x_1, \dots, x_d\}$, there exists a $h \in \Hcal((x_1, \dots, x_d), y^{\star})$ such that $x \notin \operatorname{supp}(h, y^{\star})\setminus \{x_1, \dots, x_d\}.$ Let $\hat{x}_d = \Gcal((x_1, h(x_1), y^{\star}), \dots, (x_d,  h(x_d), y^{\star}))$, and suppose without loss of generality that $\hat{x}_d \notin \{x_1, \dots, x_d\}$. Then, by the previous observation, there exists a $h^{\star} \in \Hcal((x_1, \dots, x_d), y^{\star})$ for which $\hat{x}_d \notin \operatorname{supp}(h^{\star}, y^{\star}).$ Thus, we have shown that there exists a sequence of distinct examples $x_1, \dots, x_d$, a label $y \in \Ycal$, and a hypothesis $h \in \Hcal$ such that $h(x_i) = y$ for all $i \in [d]$, and 

$$\Gcal((x_1, h(x_1), y), \dots, (x_d,  h(x_d), y)) \notin  \operatorname{supp}(h, y) \setminus \{x_1, \dots, x_d\}.$$ The proof is complete after noting that $n$ was arbitrary, and thus this holds for all $n \in \mathbbm{N}.$ 
\end{proof}

\subsection{Proof sketches for Theorem \ref{thm:mcnonunifgen}}

The proof of the necessity direction follows identically to that of Theorem \ref{thm:nonunifgen}, so we omit that proof and only prove the sufficiency direction. Although one can prove the sufficiency direction in Theorem \ref{thm:mcnonunifgen} through a reduction to prompted uniform generation (like we did for Theorem \ref{thm:nonunifgen}), we provide a more direct proof using the prompted closure dimension to avoid repetition. 

\begin{proof} (sketch of sufficiency)
Suppose $\Hcal \subseteq \Ycal^{\Xcal}$ is a hypothesis class satisfying the $\operatorname{PUUS}$ property such that there exists a non-decreasing sequence of classes $\Hcal_1 \subseteq \Hcal_2 \subseteq \dots$ with $\Hcal =\bigcup_{i\in\naturals} \Hcal_i$ and $\operatorname{PC}( \Hcal_n)<\infty$ for every $n\in\naturals$.  First note that $\operatorname{PC}(\Hcal_{n})$ is monotonic increasing in $n$. We consider two cases: $\lim_{n\to\infty}\operatorname{PC}(\Hcal_{n})=\infty$ and $\lim_{n\to\infty}\operatorname{PC}(\Hcal_{n})<\infty$.

In the first case, consider the following generator $\Gcal$. Fix $t \in \mathbbm{N}$ and consider any finite sequence of tuples $(x_1, p_1, y_1), \dots, (x_t, p_t, y_t) \in (\Xcal, \Ycal, \Ycal)^{t}$. $\Gcal$ extracts $B_t:=\{x_i:p_i=y_t\}$. Let $d_t := |B_t|$ be the number of unique examples whose label is $y_t$. $\Gcal$ first computes 

$$n_t = \max\{n \in \mathbbm{N} : \operatorname{PC}(\Hcal_{n}) < d_t\}\cup\{0\}.$$

If $n_t = 0$, meaning $\operatorname{PC}(\Hcal_{1})\ge d_t$,  $\Gcal$ plays any $\hat{x}_t \in \Xcal$. If $n_t >0$ but $|\Hcal_{n_t}(B_t,y_t)|=0$, $\Gcal$ also plays any $\hat{x}_t\in\Xcal$.  If $n_t > 0$ and $|\Hcal_{n_t}(B_t,y_t)| \geq 1$, $\Gcal$ plays any 


$$\hat{x}_t \in \langle B_t,y_t\rangle_{\Hcal_{n_t}} \setminus\{x_1,\dots,x_t\}.$$

We now prove that such a $\Gcal$ is a non-uniform generator for $\Hcal$. To that end, 
let $h^{\star}$ be the hypothesis chosen by the adversary and suppose that $h^{\star}$ belongs to $\Hcal_{n^{\star}}$. Let $d^{\star} := \operatorname{PC}(\Hcal_{n^{\star}}).$ We show that for a label sequence $(x_1,h^{\star}(x_1),y_1),\dots,(x_t,h^{\star}(x_t),y_t)$, such that $d_t := |\{x_i:h^{\star}(x_i)~=~y_t\}| \geq d^{\star} + 1$,  we have
$$\Gcal((x_1,h^{\star}(x_1),y_1),\dots,(x_t,h^{\star}(x_t),y_t)) \in \operatorname{supp}(h^{\star},y_t)\setminus\{x_1,\dots,x_t\}.$$
By definition, $\Gcal$ first computes

$$n_t = \max\{n \in \mathbbm{N} : \operatorname{PC}(\Hcal_{n}) < d_t\} \cup \{0\}.$$

Note that $n_t\ge n^{\star}$ since $\operatorname{PC}(\Hcal_{n^{\star}}) = d^{\star} < d^{\star} + 1 \leq d_t$. Thus, $|\Hcal_{n_t}(B_t,y_t)| \geq 1$ since $h^{\star} \in \Hcal_{n_t}$, where $B_t=\{x_i:h^{\star}(x_i)=y_t\}$. Accordingly, by construction of $\Gcal$, we have that it computes


$$V_t := \langle B_t,y_t\rangle_{\Hcal_{n_t}},$$
and plays any $\hat{x}_t\in V_t \setminus\{x_1,\dots,x_t\}$. The proof is complete by noting that $h^{\star} \in \Hcal_{n_t}$ and $d_t \geq \operatorname{PC}(\Hcal_{n_t}) + 1$ which gives that $|V_t|=\infty$ and $V_t \subseteq \operatorname{supp}(h^{\star}).$

In the second case, suppose $\lim_{n\to\infty}\operatorname{PC}(\Hcal_{n}) := c<\infty$. Consider the following generator~$\Gcal$. Fix $t \in \mathbbm{N}$ and consider any finite sequence of tuples $(x_1, p_1, y_1), \dots, (x_t, p_t, y_t) \in (\Xcal, \Ycal, \Ycal)^{t}$, $\Gcal$ extracts $B_t=\{x_i:p_i=y_t\}$. Let $d_t := |B_t|$ be the number of unique examples whose label is $y_t$. If $d_t<c$ or $|\Hcal_{d_t}(B_t,y_t)|=0$, $\Gcal$ plays any $\hat{x}_t\in\Xcal$. Otherwise, if $d_t \ge c$ and $|\Hcal_{d_t}(B_t,y_t)|>0$, $\Gcal$ plays any 


$$\hat{x}_t \in \langle B_t, y_t\rangle_{\Hcal_{d_t}} \setminus\{x_1,\dots,x_t\}.$$

Let $h^{\star}$ be the hypothesis chosen by the adversary and suppose $h^{\star}$ belongs to $\Hcal_{n^{\star}}$. We show that for every labeled sequence $(x_1,h^{\star}(x_1),y_1),\dots,(x_t,h^{\star}(x_t),y_t)$ such that $d_t :=|\{x_i:h^{\star}(x_i)~=~y_t\}|> \max(c,n^{\star})$,  we have 
$$\Gcal((x_1,h^{\star}(x_1),y_1),\dots,(x_t,h^{\star}(x_t),y_t)) \in \operatorname{supp}(h^{\star},y_t)\setminus\{x_1,\dots,x_t\}. $$
Because $d_t \geq n^{\star}$, we have that $h^{\star} \in \Hcal_{d_t}$. Therefore, by construction, $\Gcal$ computes 


$$V_t := \langle \{x_i:h^{\star}(x_i)=y_t\},y_t\rangle_{\Hcal_{[d_t]}},$$
and plays any $\hat{x}_t\in V_t\setminus\{x_1,\dots,x_t\}$. The proof is complete after noting that $|V_t|=\infty$ and $V_t\subseteq \operatorname{supp}(h^{\star},y_t)$ using the fact that $d_t>\max(c,n^{\star}).$
\end{proof}

\section{Weaker Sufficiency Conditions for Generatability in the Limit} \label{app:weaksuff}

In this section, we prove a weaker sufficiency condition than the one in Theorem \ref{thm:geninlim} for generatability in the limit. Before we present the main result, we define a new property of a class termed the Eventually Unbounded Closure property. 

\begin{definition}[Eventually Unbounded Closure] \label{def:uuc} A class $\Hcal \subseteq \{0, 1\}^\Xcal$ has the \emph{Eventually Unbounded Closure (EUC)} property if for every $h \in \Hcal$ and any enumeration of its support $x_1, x_2, \dots$, there exists a $t \in \naturals$ such that $|\langle x_1, x_2, \dots, x_t \rangle_{\Hcal}| = \infty$
\end{definition}

Because the closure with respect to a class $\Hcal$ does not depend on any one particular hypothesis, Definition \ref{def:uuc} is really only a stream dependent property. That is, an equivalent representation of Definition \ref{def:uuc} is as follows -- $\Hcal$ satisfies the EUC property if and only if for every sequence of examples $x_1, x_2, \cdots$ there exists a $t\in \naturals$ such that $|\langle x_1, \dots, x_t\rangle_{\Hcal}| = \infty$ or $\langle x_1, \dots, x_t\rangle_{\Hcal} = \bot.$ 


The EUC property is sufficient for generation in the limit. Indeed, just consider the generator that plays arbitrarily until the closure is unbounded, after which it only plays from this infinite set.  Since the EUC property guarantees that the closure will become infinite in finite time, this is a valid generator. Moreover, note that such a generator is eventually auto regressive -- once the closure is infinite in size, the generator no longer needs to observe positive examples to generate new, unseen examples in the future. One might be tempted to think that the EUC property is also necessary for generatability in the limit. However, the following lemma shows that this is not the case. For an infinite bit string $b \in \{0, 1\}^{\mathbbm{N}}$, let $|b|$ denote the number of $1$'s. 

\begin{lemma}\label{lem:stillnotnec} Let $\Xcal$ be countable. There exists a class $\Hcal \subseteq \{0, 1\}^{\Xcal}$ such that $\Hcal$ satisfies the \emph{UUS} property, is non-uniformly generatable, but does not satisfy the \emph{EUC} property. 
\end{lemma}

\begin{proof} Let $\Xcal = \mathbbm{N}$ and $\{p_n\}_{n \in \naturals}$ be the sequence of prime numbers. Consider the class $\Hcal := \{h_b: b \in \{0, 1\}^{\naturals}, |b| < \infty\}$ where $h_b$ is defined such that $\operatorname{supp}(h_b) := \{p^{1 + \sum_{i=1}^n b_i}_n\}_{n \in \naturals}.$ Note that $\Hcal$ satisfies the UUS property. Moreover, $\Hcal$ is countable since the collection of all countably infinite bit strings with finite size is countable. Thus, by Corollary \ref{cor:countnonunif}, $\Hcal$ is non-uniformly generatable. Finally, to see that $\Hcal$ does not satisfy the EUC property, observe that for every finite sequence of prime powers, its closure can only contain this sequence itself. To see why, note that for any prime not in the sequence and any power, one can always construct a hypothesis which contains the finite sequence, but not that prime power. 
\end{proof}

 That said, we can use the EUC property to weaken our sufficiency condition in Theorem \ref{thm:geninlim} by replacing finite Closure dimension with the EUC property.  

\begin{theorem} \label{thm:weaksuff}  Let $\Xcal$ be countable and $\Hcal \subseteq \{0, 1\}^{\Xcal}$ be any class satisfying the \emph{UUS} property. If there exists a finite sequence of classes $\Hcal_1, \Hcal_2, \dots, \Hcal_n$, all satisfying the \emph{EUC} property, such that $\Hcal = \bigcup_{i = 1}^n \Hcal_i$, then $\Hcal$ is generatable in the limit. 
\end{theorem}

Theorem \ref{thm:weaksuff} replaces the constraint that each of the finite number of sub-classes need to be uniformly generatable with the constraint that they need to satisfy the EUC property. This is a weakening as uniform generatability implies EUC but not the other way around. The proof is effectively the same as the proof of Theorem \ref{thm:geninlim} with the only difference being that the amount of time before which the generator computes closures is now stream dependent. 

\begin{proof} (sketch of Theorem \ref{thm:weaksuff})
Let $\Hcal = \bigcup_{i = 1}^n \Hcal_i$ be such that $\Hcal_i$ satisfies the EUC property for all $i \in [n].$ 
Consider the following generator $\Gcal$. On a valid input sequence $x_1, x_2, \dots$, let $t_i \in \naturals$ be the smallest time point such that either $\langle x_1, \dots, x_{t_i}\rangle_{\Hcal_i} = \bot$ or $|\langle x_1, \dots, x_{t_i}\rangle_{\Hcal_i}| = \infty$ for all $i \in [n].$ Note that such an $t_i \in \naturals$ must exist because $\Hcal_i$ satisfies the EUC property.  Let $t^{\star} = \max_{i \in [n]} t_i$ be the largest time point. $\Gcal$ plays arbitrarily up to, but not including, time point $t^{\star}$. On time point $t^{\star}$, $\Gcal$ computes $\langle x_1, \dots, x_{t_i} \rangle_{\Hcal_i}$ for all $i \in [n]$. Let $S \subseteq [n]$ be the subset of indices such that $i \in S$ if and only if $\langle x_1, \dots, x_{t_i} \rangle_{\Hcal_i} \neq \bot.$ For every $i \in S$, let $(z^{(i)}_j)_{j \in \mathbbm{N}}$ be the natural ordering of  $\langle x_1, \dots, x_{t_i} \rangle_{\Hcal_i}$, which is guaranteed to exist since $\Xcal$ is countable. For every $t \geq t^{\star}$, valid sequence of revealed examples $x_1, \dots, x_t$, and $i \in S$, $\Gcal$ computes 
\begin{equation}
n_t^i := \max\{n \in \mathbbm{N} : \{z^{(i)}_{1}, \dots, z^{(i)}_{n}\}  \subset \{x_1, \dots, x_t\}\}
\end{equation}

\noindent and $i_t\in \argmax_{i \in S} n_t^i.$ Finally, $\Gcal$ plays any $\hat{x}_t \in \langle x_1, \dots, x_{t_i} \rangle_{\Hcal_{i_t}} \setminus \{x_1, \dots, x_t\}.$ The rest of the proof is identical to that of Theorem \ref{thm:geninlim}.
\end{proof}

Due to its connection to autoregressive generation in the limit, understanding which classes satisfy the EUC property is an interesting property on its own. For example, while it is clear that the EUC property is weaker than uniform generatability but stronger than generatability in the limit, its relationship to non-uniform generatability is less clear. Lemma \ref{lem:stillnotnec} shows that if one restricts to countable classes, then the EUC property is strictly stronger than non-uniform generatability -- there exists a countable class which does not have the EUC property. We leave as an open question whether the EUC property is strictly stronger than non-uniform generatability even amongst uncountable classes.

\begin{question}
    Does there exists an uncountable class $\Hcal$ which is non-uniformly generatable, but does not satisfy the \emph{EUC} property?
\end{question}

We can also provide a sufficient condition for generatability in the limit akin to that of non-uniform generatability in terms of the EUC property. 

\begin{theorem} \label{thm:altweaksuff}  Let $\Xcal$ be countable and $\Hcal \subseteq \{0, 1\}^{\Xcal}$ be any class satisfying the \emph{UUS} property. If there exists a non-decreasing sequence of classes $\Hcal_1 \subseteq \Hcal_2 \subseteq \dots$, all satisfying the \emph{EUC} property, such that $\Hcal = \bigcup_{i = 1}^{\infty} \Hcal_i$,  then $\Hcal$ is generatable in the limit. 
\end{theorem}

Note the sufficiency condition in Theorem \ref{thm:altweaksuff} is weaker than the sufficiency condition for non-uniform generation as uniform generatability is stronger than EUC. Our proof of Theorem \ref{thm:altweaksuff} is constructive -- we give an algorithm which generates in the limit as long as the sufficiency condition is met. The algorithm can be thought of as a generalization of the algorithm by \cite{kleinberg2024language} for countable classes. The high-level idea is to play from the closure of rightmost class whose closure is infinite in size. 

\begin{algorithm}
\setcounter{AlgoLine}{0}
\caption{Generator}\label{alg:gen}
\KwIn{Hypothesis class $\Hcal = \bigcup_{n = 1}^{\infty} \Hcal_n$ such that $\Hcal_1 \subseteq \Hcal_2 \subseteq \dots$ and $\Hcal_n$ satisfies EUC for all $n \in \naturals$}

\For{$t = 1,2, \dots$} {
    Adversary reveals positive example $x_t$

    Let $n_t = \max \, \{n \in [t]: |\langle x_1,\dots, x_t \rangle_{\Hcal_{n}}| = \infty\} \cup \{0\}$

    \uIf{$n_t = 0$}{
        Play arbitrarily from $\Xcal$
    }

    \uElse{
        Play arbitrarily from $\langle x_1,\dots, x_t \rangle_{\Hcal_{n_t}} \setminus \{x_1, \dots, x_t\}$
    }
}
\end{algorithm}

\begin{proof} We will show that Algorithm \ref{alg:gen} generates in the limit. To that end, let $h^{\star} \in \Hcal$ be the hypothesis and $x_1, x_2, \dots$ be an enumeration of $\operatorname{supp}(h^{\star})$ chosen by the adversary.  Let $n^{\star} \in \naturals$ be the smallest number such that $h^{\star} \in \Hcal_{n^{\star}}$. For every $n \in \naturals$, since $\Hcal_{n}$ satisfies the EUC property, there exists a $t_n \in \mathbbm{N}$ such that either $|\langle x_1, \dots, x_{t_n} \rangle_{\Hcal_{n}}| = \infty$ or $\langle x_1, \dots, x_{t_n} \rangle_{\Hcal_{n}} = \bot.$ 


We claim that for all $n \geq n^{\star}$, we have that $|\langle x_1, \dots, x_{s}\rangle_{\Hcal_{n}}| = \infty$ and $\langle x_1, \dots, x_{s}\rangle_{\Hcal_{n}} \subseteq \operatorname{supp}(h^{\star})$ for all $s \geq t_n.$ Fix some $n \geq n^{\star}$. By definition, we know that $h^{\star} \in \Hcal_{n}.$ In addition, since the stream $x_1, x_2, \dots$ is an enumeration of $h^{\star}$ and $\Hcal_{n}$ satisfies the EUC property, it must be the case that $|\langle x_1, \dots, x_{t_{n}}\rangle_{\Hcal_{n}}| = \infty.$ Moreover, $\langle x_1, \dots, x_{t_{n}}\rangle_{\Hcal_{n}} \subseteq \operatorname{supp}(h^{\star})$ because $h^{\star} \in \Hcal_{n}.$ Now, fix some $s \geq t_n$. Then, it must be the case that $\langle x_1, \dots, x_{s}\rangle_{\Hcal_{n}} \supseteq \langle x_1, \dots, x_{t_n}\rangle_{\Hcal_{n}}$ and therefore $|\langle x_1, \dots, x_{s}\rangle_{\Hcal_{n}}| = \infty$.  Because $h^{\star} \in \Hcal_{n}$, it also must be the case that $\langle x_1, \dots, x_{s}\rangle_{\Hcal_{n}} \subseteq \operatorname{supp}(h^{\star})$. This completes the proof of the claim as $n \geq n^{\star}$ was chosen arbitrarily. 

Now, we complete the overall proof of Theorem \ref{thm:altweaksuff} by showing that Algorithm \ref{alg:gen} generates perfectly on and after round $t_{n^{\star}}.$ On round $t = t_{n^{\star}}$, $n_t = n^{\star}$, and thus by Line 7, Algorithm \ref{alg:gen} generates from $\operatorname{supp}(h^{\star})\setminus \{x_1, \dots, x_t\}$ since $\langle x_1, \dots, x_{t}\rangle_{\Hcal_{n^{\star}}} \subseteq \operatorname{supp}(h^{\star}).$ Now fix a round $t > t_{n^{\star}}$. Then, observe that by the claim above we have that $n_t \geq n^{\star}$ and $h^{\star} \in \Hcal_{n_t}.$ Accordingly, we have that $\langle x_1, \dots, x_{t}\rangle_{\Hcal_{n_t}} \subseteq \operatorname{supp}(h^{\star})$ and therefore by Line 7, Algorithm \ref{alg:gen} plays from $\operatorname{supp}(h^{\star})\setminus \{x_1, \dots, x_t\}$. Since $t > t_{n^{\star}}$ was picked arbitrarily, the proof is complete. 
\end{proof}

We note that in addition to Theorem \ref{thm:nonunifgen}, Theorem \ref{thm:altweaksuff} also recovers the result by \cite{kleinberg2024language} that all countable classes are generatable in the limit. 

\end{document}